\documentclass{article}

\PassOptionsToPackage{numbers, compress}{natbib}

\usepackage[main, final]{neurips_2026}

\usepackage[utf8]{inputenc} 
\usepackage[T1]{fontenc}    
\usepackage{hyperref}       
\usepackage{url}            
\usepackage{booktabs}       
\usepackage{amsfonts}       
\usepackage{nicefrac}       
\usepackage{microtype}      
\usepackage{xcolor}         

\usepackage{amsmath}
\usepackage{amssymb}
\usepackage{mathtools}
\usepackage{amsthm}
\usepackage{ulem}
\usepackage{comment}
\usepackage{capt-of}
\usepackage{wrapfig}

\usepackage[capitalize,noabbrev]{cleveref}

\theoremstyle{plain}
\newtheorem{theorem}{Theorem}[section]

\newtheorem{lemma}[theorem]{Lemma}

\theoremstyle{definition}

\newtheorem{assumption}[theorem]{Assumption}
\theoremstyle{remark}

\title{Revealing Epistemic Uncertainty in MLLMs via Causal-Invariant Masking}

\author{%
Haoyang Luo\\
Computer Science\\
City University of Hong Kong\\
\And
Linwei Tao \\
Machine Learning Research\\
Apple\\
\texttt{}
\AND
Jie Gui \\
Cyber Science and Engineering\\
Southeast University\\
\texttt{}
\And
Xinghao Chen\\
Huawei Noah's Ark Lab\\
\texttt{}
\AND
Chang Xu\\
Computer Science\\
University of Sydney\\
\texttt{}
\And
Jianyuan Guo\\
Computer Science\\
City University of Hong Kong\\
\texttt{}
\And
Minjing Dong\\
Computer Science\\
City University of Hong Kong\\
\texttt{}
}

\begin{document}

\maketitle

\begin{abstract}
Multimodal Large Language Models (MLLMs) suffer from hallucinations, creating a critical need for Uncertainty Quantification (UQ) to ensure reliable deployment. However, existing approaches struggle to detect uncertainty caused by superficial associations, especially when the query-relevant signal is weak. We mainly attribute this issue to their bias toward aleatoric uncertainty arising from data ambiguity, overlooking epistemic uncertainty stemming from model limitations. To further decompose uncertainty types for a comprehensive UQ, we propose Causal-Invariant Masking (\textbf{CIM}), which measures the semantic shift between the original predictions and those conditioned on a causally-focused view. Based on this framework, we introduce Semantic Divergence as our core metric for UQ and provide theoretical evidence that it converges to the variance of model's sensitivity to non-causal correlations, establishing its ability to capture MLLM's limitation. To accelerate UQ in MLLMs, we further propose Expected Embedding Drift (EED), a fast geometric proxy metric that estimates semantic shift directly within the hyperspherical embedding space. Experiments show that our method achieves state-of-the-art performance on various benchmarks, while the proposed EED accelerates by nearly 50\% with comparable performance.
\end{abstract}

\section{Introduction}
\label{sec:intro}

Large Language Models (LLMs)~\cite{zhao2023survey,minaee2024large} and Multimodal Large Language Models (MLLMs)~\cite{yin2024survey,bai2024survey} have significantly altered the landscape of artificial intelligence.
By connecting visual perception with textual semantics, MLLMs are able to demonstrate consistent success in tasks ranging from visual question answering to complex reasoning~\cite{liu2023visual,peng2025mv,zhang2025enhancing}. 
However, their reliable deployment in real-world applications requires solving the severe problem of hallucination~\cite{bai2024hallucination,huang2025survey}, where MLLMs confidently generate factually incorrect but grammatically plausible content.
Especially for VLMs, non-existent objects, attributes, and events may unreasonably occur in the responses when misleading factors, either in vision or language, dominate the query.
It suggests that MLLMs do not internally realize their visual capability boundary and highlights a critical need for robust safety measures targeting generative trustworthiness.

With this aim, numerous Uncertainty Quantification (UQ) methods have been developed. By estimating the confidence level of a model, systems can abstain from answering or flag potential hallucinations when encountering high uncertainty. Therefore, precise UQ becomes a proxy for hallucination detection.
Traditional UQ approaches typically measure the entropy of the predictive distribution, often aggregated at the token level~\cite{xiongcan, tao2025revisiting}. Recent advancements~\cite{kuhnsemantic, kossen2024semantic, nikitin2024kernel} have shifted towards measuring \textit{semantic uncertainty} using sampling-based methods, such as Semantic Entropy (SE)~\mbox{\cite{farquhar2024detecting}}, which aggregates probabilities over clusters of similar generations to calculate the response variety. 

However, these approaches exhibit limited discrimination under hallucinative visual content, especially when the query-irrelevant elements generate spurious hints. For instance, background objects with prominent colors can dominate the response to queries on the foreground object, making consistent mistakes in the sampled output.
Thus, a more comprehensive UQ is necessary for MLLMs. In fact, both \textit{epistemic} and \textit{aleatoric} uncertainties play important roles in UQ, where epistemic uncertainty reveals insufficiency in the model's perception or comprehension and aleatoric uncertainty reflects the poor data quality or incomplete information~\cite{hullermeier2021aleatoric}. Quantifying both of them is crucial in UQ~\cite{shorinwa2025survey, ahdritz2024distinguishing}, however, existing approaches fail to do so because they typically assume that all forms of uncertainty can be equally estimated. In this paper, we show that they are biased towards quantifying mostly aleatoric uncertainty in MLLMs, leaving epistemic uncertainty underestimated.

We introduce Causal-Invariant Masking (\textbf{CIM}), a framework that disentangles intrinsic epistemic instability from data ambiguity by probing semantic consistency under causally invariant transforms. Our core insight is that robust reasoning must remain invariant when task-irrelevant features are removed. By constructing a Region-of-Interest (RoI) view that retains only causal features, we measure the shift in the model's belief distribution. We quantify this using Semantic Divergence, defined as the shift from the semantic distributions of the original to the focused inputs. 
Supported by theoretical derivation, this metric targets the model stability against irrelevant distractions to effectively isolate epistemic uncertainty, thereby alleviating the biased quantification in previous works. 
To address the computational costs of discrete semantic clustering, we then propose a fast proxy called Expected Embedding Drift (EED). Leveraging the hyperspherical geometry of the embedding space, EED estimates semantic shift directly from token embeddings, eliminating costly discrete operations. Experiments demonstrate that our method significantly outperforms baselines, achieving state-of-the-art hallucination detection of \textbf{76.8\%} AUROC on the challenging AdVQA and superior results on POPE. Furthermore, EED reduces computational overhead by ~\textbf{50\%} without compromising performance, facilitating applications such as hallucination mitigation and adaptive-resolution inference.

\section{Related Work}
\label{sec:rel_work}

\subsection{Uncertainty Estimation in LLMs}

UQ in LLMs is pivotal for reliability and hallucination detection, generally categorized into intrinsic, verbalized, and consistency-based methods. Intrinsic metrics like perplexity~\cite{shorinwa2025survey} measure token-level confidence, while verbalized methods~\cite{xiong2023can, tian2023just} prompt models to explicitly express uncertainty, often yielding better calibration than raw probabilities. Consistency-based works are currently dominant approaches. Semantic Entropy~\cite{farquhar2024detecting} and Semantic Dispersion~\cite{lin2024generating} detect hallucinations by aggregating multiple sampled responses based on meaning rather than exact phrasing. Recent works further exploit internal representations as uncertainty indicators by utilizing activation covariance and propose EigenScore~\cite{cheninside} or geometric convex hull analysis~\cite{catak2024uncertainty} to quantify dispersion in high-dimensional embedding spaces. However, these methods target general UQ for LLMs and can be challenged in vision-language tasks.

\subsection{Uncertainty Estimation in MLLMs}
Extending UQ to MLLMs involves challenges from combined visual and textual modalities, often focusing on hallucination detection in vision-language tasks. Similar to those for LLMs, consistency-based sampling techniques are explored, \textit{e.g.},~\cite{lau2025uncertainty} develops UMPIRE, a training-free UQ framework that estimates consistency in multiple sampled responses by a coherence-adjusted semantic volume. Perturbation-based methods introduce variations in inputs to probe uncertainty. VL-Uncertainty~\cite{zhang2024vl} and Uncertainty-o~\cite{zhang2025uncertainty} measure hallucination in LVLMs by computing entropy over clustered responses, but using semantic-equivalently perturbed visual and textual queries, supporting tasks of hallucination mitigation and uncertainty-aware reasoning. However, these approaches typically rely on \textit{global} perturbations under the full context, without explicitly isolating the model's reliance on critical and irrelevant features. CIM is also a perturbation-style test in its implementation, but its perturbation is query-conditioned and RoI-preserving. The goal is to compare the original belief with a view that keeps the answer evidence while suppressing irrelevant context, rather than purely testing variance under general information loss.
\section{Method}
\label{sec:method}

\subsection{Preliminary}
\label{sec:preliminary}

\begin{figure*}[htbp]
\vskip 0.2in
\begin{center}
\centerline{
\hfill
\includegraphics[width=0.33\textwidth]{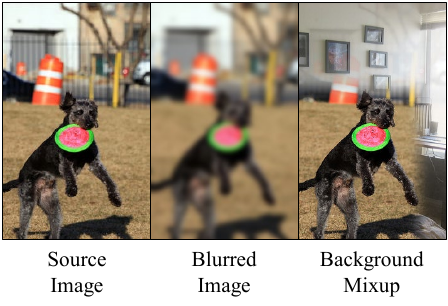}
\hfill
\includegraphics[width=0.335\textwidth]{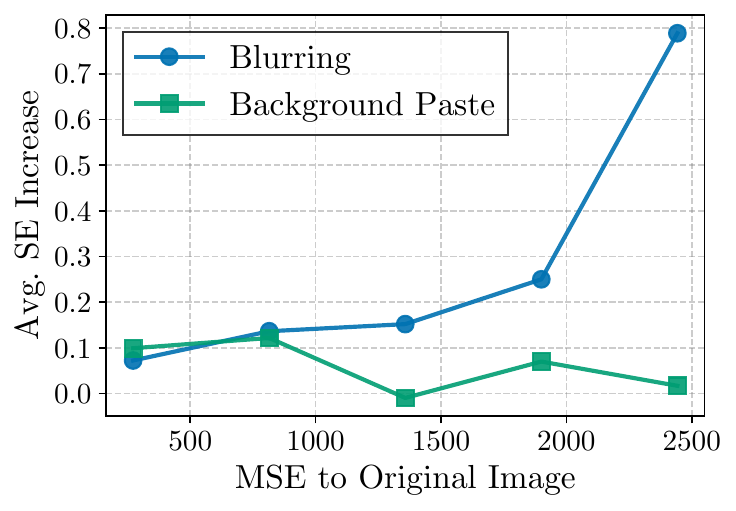}
\hfill
\includegraphics[width=0.313\textwidth]{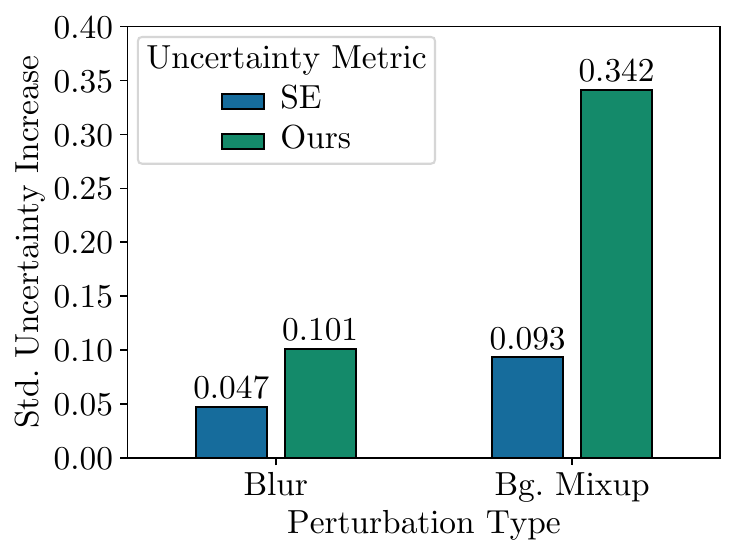}
\hfill
}
\caption{\textbf{Left:} A case of blurred and background-mixup image pair. \textbf{Middle:} Sensitivity of Semantic Entropy to visual perturbations in MLLMs. Similar levels of error to the original image are binned into the same point. \textbf{Right:} Comparison of uncertainty increase under blurring and background mixup perturbations. While Semantic Entropy increases slightly across both, it fails to distinguish hallucination. In contrast, our proposed Semantic Divergence remains robust to visual blur but exhibits a distinct spike under hallucinative perturbations. The values are normalized by the standard deviation of each metric.}
\label{fig:combined_perturbation_analysis}
\end{center}
\vskip -0.2in
\end{figure*}

We consider an MLLM parameterized by $\theta$ over visual space $\mathcal{V}$ and vocabulary $\mathcal{W}$. Given an image $v$ and query $q$, it generates a token sequence $y = (y_1, \dots, y_T)$ autoregressively via $P_\theta(y | v, q) = \prod_{t=1}^T P_\theta(y_t | y_{<t}, v, q)$. For simplicity, we henceforth omit $q$ and $\theta$, denoting the distribution as $P(y|v)$.
Token-level entropy, \textit{i.e.}, $\mathcal{H}(y) = -\sum P(y)\log P(y)$, or perplexity values, \textit{i.e.}, $\text{PPL}(y) = \exp\left(-\frac{1}{T}\sum_{t=1}^T \log P(y_t | y_{<t}, v)\right)$, often misestimate uncertainty because semantically equivalent outputs are treated as distinct. Semantic Entropy (SE)~\cite{farquhar2024detecting} samples $K$ sequences, maps them to a semantic space $\mathcal{S}$ with a clustering operation $\mathcal{G}: \mathcal{W}^* \to \mathcal{S}$, and aggregates meaning-level probabilities: $\mathcal{H}_{\text{sem}}(v) = - \sum_{s \in \mathcal{S}} P(s|v) \log P(s|v)$, where $\mathcal{W}^*$ is the set of $\mathcal{W}$-strings and $P(s|v) = \sum_{y: \mathcal{G}(y)=s} P(y|v)$ is the aggregated probability.
While SE handles linguistic variance, it still fails to distinguish epistemic from aleatoric uncertainty.

\subsection{Epistemic Uncertainty Isolation via Causal Masking}
Reliable MLLM deployment requires separating aleatoric uncertainty, which reflects data ambiguity, from epistemic uncertainty, which reflects model limitation. Existing metrics often blend the two and therefore skew toward aleatoric noise. As illustrated in Fig.~\ref{fig:combined_perturbation_analysis} middle, Semantic Entropy rises quickly under blurring for aleatoric uncertainty, but is much less sensitive to background mixup for epistemic uncertainty, suggesting that it treats causal and non-causal features alike. Robust reasoning instead should satisfy causal invariance: predictions should stay stable when task-irrelevant context changes, whereas uncertain responses should not.

\begin{figure*}[t]
\vskip 0.2in
\begin{center}
\centerline{
\includegraphics[width=0.99\textwidth]{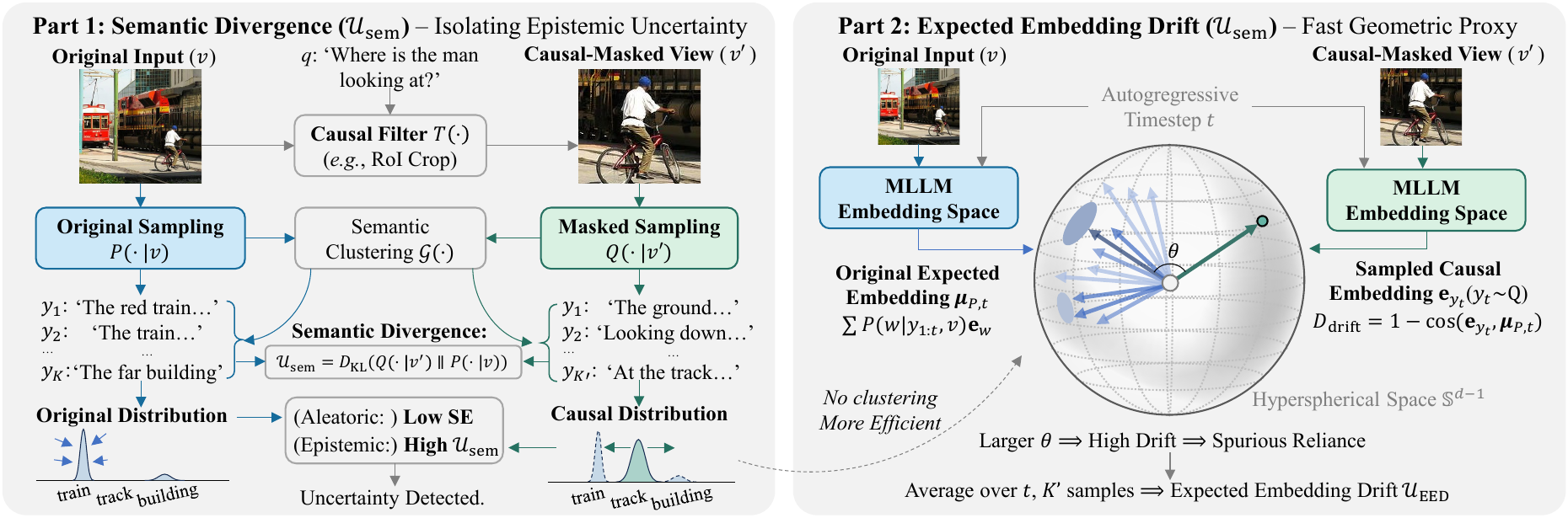}
}
\caption{\textbf{The detailed framework of CIM.} Left: We isolate epistemic uncertainty via Semantic Divergence, which measures the KL divergence between the semantic distribution of the original input and a causally-filtered view $v'$, revealing the model's reliance on query-irrelevant information. Right: Expected Embedding Drift geometrically estimates semantic shift via cosine distance in the hyperspherical embedding space, eliminating the computational overhead of discrete clustering.}
\label{fig:main_figure}
\end{center}
\vskip -0.2in
\end{figure*}

\paragraph{Methodology.}
We propose \textbf{CIM} to isolate the epistemic component through \textbf{causal invariance}.
Given a query $q$, we decompose the visual input $v$ into causal features $v_c$ and non-causal features $v_s$, where $v_c$ is query-relevant, while $v_s$ may carry spurious hints of correlations.
This split is defined with respect to an oracle predictor satisfying $P(s|v_c,v_s)=P(s|v_c)$, rather than assuming that the evaluated MLLM already knows it.
Accordingly, reliable reasoning should remain stable when non-causal variables are removed, and any large shift after filtering $v_s$ exposes reliance on spurious correlations.
We therefore introduce a $v_c$-preserving masking operator $T(\cdot): \mathcal{V} \to \mathcal{V}$, mainly implemented as query-conditioned Region-of-Interest (RoI) cropping. In practice, we detect key objects with GroundingDINO~\cite{liu2024grounding}, filter overly large boxes, and crop the union bounding box so inter-object relations and local context are preserved. This acts as an empirical proxy for the oracle $v_c$ filter: it preserves query-relevant evidence while removing as much $v_s$ as possible, yielding a focused view $v'=T(v)$. Because only a $v_c$-preserving transform is required, the detector choice is not essential, as validated in Appendix~\ref{sec:florence_detector}.

To quantify the shift caused by causal masking, we define \textbf{Semantic Divergence} (\textbf{SD}) as the Kullback-Leibler divergence between the semantic probability distributions of the intervened input $P(s|v', q)$ and the original model $P(s|v, q)$.
In practice, we sample $K'$ sequences from the model, group them into clusters $\{s_i\}_{j=1}^{M}$ via entailment $\mathcal{G}(\cdot)$, and aggregate their likelihoods similarly to SE.
After grouping, the Semantic Divergence is thus calculated as:
\begin{align} \label{eq:divergence}
\mathcal{U}_{\text{sem}}(v, q) &= D_{KL}(P(s|v') \| P(s|v)) \nonumber\\
&= \sum_{s \in \mathcal{S}} P(s|v') \log \frac{P(s|v')}{P(s|v)}.
\end{align}
This KL divergence measures the \textit{information gain} from removing non-causal features.
If the model is robust, or the data is ambiguous while the model is honest, the distribution over meanings remains stable, yielding low $\mathcal{U}_{\text{sem}}$.
In contrast, hallucinations that rely on background context cause a semantic shift and high $\mathcal{U}_{\text{sem}}$. Unlike global perturbations that flatten both distributions, this measure targets epistemic failures. Our $\mathcal{U}_{\text{sem}}$ significantly outperforms traditional SE in recognizing epistemic uncertainty. As illustrated in Fig. \ref{fig:combined_perturbation_analysis} right, we plot the shift of uncertainty when imposing different ambiguity types. Our $\mathcal{U}_{\text{sem}}$ significantly outperforms traditional SE in recognizing epistemic uncertainty.

\subsection{Theoretical Analysis}
We justify why $\mathcal{U}_{\text{sem}}$ isolates epistemic uncertainty by framing the MLLM's semantic output as an Energy-Based Model (EBM)~\cite{carreira2005contrastive, lecun2006tutorial, du2019implicit}. This is compatible with autoregressive decoding: prior work interprets pre-softmax logits as token energies~\cite{minut2026spilled}, connects left-to-right language models and EBMs through a function-space equivalence~\cite{blondel2025autoregressive}, and uses energy functions for constrained text decoding~\cite{qin2022cold}. Because MLLM decoding remains autoregressive after conditioning on images, the same energy view applies. Under standard causal assumptions, the divergence in Eq.~\eqref{eq:divergence} converges to the variance of sensitivity to spurious features.

Standard causal theory~\cite{pearl2000models, pearl2016causal} models the input $v$ as causal variables $v_c$ and non-causal variables $v_s$. An oracle predictor would satisfy $P(s|v_c, v_s) = P(s|v_c)$, but real models often fail to be invariant. We therefore introduce a unified assumption on the masking operation and the energy landscape:

\begin{assumption}[Energy Decomposition]
\label{ass:energy_unified}
The MLLM's semantic output over $\mathcal{S}$ is modeled as a Boltzmann distribution $P_\theta(s|v) = \frac{1}{Z(v)} e^{-E_\theta(s, v)}$, where the total energy decomposes additively into a causal term $E_c(s, v_c)$ and a non-causal term $E_s(s, v_s)$. We assume that the causal energy $E_c$ remains invariant under the causal-preserving $T(\cdot)$, \textit{i.e.}, $E_c(s, v'_c) = E_c(s, v_c)$, while the non-causal energy shifts by a quantity $\Delta_s(s) \triangleq E_s(s, v_s) - E_s(s, v'_s)$.
\end{assumption}

This invariant/variant split follows the broader causal-invariance view that causal mechanisms should remain stable across interventions or environments~\cite{peters2016causal,arjovsky2019invariant,scholkopf2021toward}. Related decomposition work also separates shared invariant factors from domain-specific variants~\cite{bousmalis2016domain}.
Rather than requiring pixel-level separation, Assumption~\ref{ass:energy_unified} only asks that $T(\cdot)$ preserve query-relevant evidence while suppressing non-causal context.
The empirical validations in Appendix~\ref{sec:energy_shift_scale}, Table~\ref{table_vp_operation_ablation}, and Appendix~\ref{sec:detector_failure_sensitivity} show that causal-energy shifts are far below the model's logit margins, while detected RoIs rarely fail and closely match ground truth RoIs.

We first show that the aleatoric component $\phi_c$ cancels in the likelihood ratio. With the empirical evidence above, the cancellation derived is a high-fidelity approximation in practice.

\begin{lemma}[Aleatoric Cancellation] \label{lemma:alertoric}
Given Assumption \ref{ass:energy_unified}, the log-likelihood ratio between the perturbed and original distributions depends solely on the non-causal energy shift $\Delta_s(s)$ and the partition function ratio, \textit{i.e.},
\begin{align}
\log \frac{Q(s)}{P(s)} = \Delta_s(s) + \log \frac{Z(v)}{Z(v')}.
\end{align}
\end{lemma}

The proof is provided in Appendix \ref{sec:proof_lemma_aleatoric}. Lemma~\ref{lemma:alertoric} shows that $E_c(s, v_c)$ cancels in the likelihood ratio, so the metric is invariant to ambiguity in $v_c$. Next, we establish the exact form of the divergence.

\begin{theorem}[Epistemic Jensen Gap] \label{theorem:jensen}
The Semantic Divergence $\mathcal{U}_{\text{sem}} = D_{KL}(Q \| P)$ is equivalent to the Jensen Gap of the exponential function acting on $\Delta_s(s)$ with respect to distribution $Q$.

Furthermore, if we assume the shift magnitude is bounded such that $\sup |\Delta_s(s)| \le \epsilon$ for a small $\epsilon > 0$.
The Semantic Divergence converges to half of the energy shift $\Delta_s(s)$ with a third-order error bound:
\begin{equation}
\label{eq:u_sem_prop_to_var}
    \mathcal{U}_{\text{sem}} = \frac{1}{2} \mathrm{Var}_{s \sim Q}[\Delta_s(s)] + \mathcal{O}(\epsilon^3).
\end{equation}
\end{theorem}
We provide the detailed proof in Appendix \ref{sec:proof_theorem_jensen}. \uline{This result shows that $\mathcal{U}_{\text{sem}}$ quantifies epistemic instability caused by irrelevant features under the stated approximation.} When the $v_c$-focused operation $T(\cdot)$ removes most of $v_s$, we have $\Delta_s(s)=E_s(s,v_s)$, so $\mathcal{U}_{\text{sem}}$ measures variance on query-irrelevant signals. Appendix~\ref{sec:energy_shift_scale} reports empirical $\Delta_s$ magnitudes, showing its bounded level in practice.

\subsection{Fast Quantification via Expected Embedding Drift}
Despite its theoretical validity, Semantic Divergence $\mathcal{U}_{\text{sem}}$ is expensive to compute. Equation~\eqref{eq:divergence} requires estimating $P(s|v) = \sum_{y: \mathcal{G}(y)=s} P(y|v)$, which in turn needs samples from $P$ and $Q$, followed by expensive LLM clustering $\mathcal{G}(\cdot)$.

However, we observe that semantic meaning is already embedded in the model's high-dimensional geometry, where tokens with similar meanings lie close together. We exploit this and propose \textbf{Expected Embedding Drift} (EED), a clustering-free proxy that bypasses discrete semantic grouping.
	   
Specifically, it measures the geometric divergence between the masked-input distribution $Q$ and the original semantic distribution $P$ by mapping token probabilities to single vectors on a hypersphere.

\paragraph{Expected Semantic Vector.}
For a response $y\sim Q$ at time step $t$ conditioned on $y_{<t}$ and $v$, rather than treating the output solely as a discrete categorical distribution, we view the token distribution as a distribution over \textit{semantic vectors} and compute the \textbf{expected embedding vector}, $\boldsymbol{\mu}_{P,t}$ as
\begin{equation}
\boldsymbol{\mu}_{P,t} = \mathbb{E}_{w \sim P(\cdot | y_{<t},v)} [\mathbf{e}_w] = \sum_{w \in \mathcal{W}} P(w | y_{<t}) \cdot \mathbf{e}_w,
\end{equation}
where $\mathbf{e}_w \in \mathbb{R}^d$ is the $L_2$-normalized embedding vector for token $w$ in $\mathcal{W}$. This expected embedding is the probability-weighted semantic center in hyperspherical space, filtering low-probability noise and pointing to the aggregate intent.

\paragraph{Drift Calculation.}
To measure epistemic inconsistency, for each generated response sequence $y = (y_1, \dots, y_T)\sim Q$, we measure the cosine drift of $y_t$ from the source semantic expectation $\boldsymbol{\mu}_{P,t}$ at each autoregressive timestep $t$:
\begin{equation} \label{eq:drift}
D_{\text{drift}}(y_t) = 1 - \cos(\mathbf{e}_{y_t}, \boldsymbol{\mu}_{P,t}) = 1 - \frac{\mathbf{e}_{y_t} \cdot \boldsymbol{\mu}_{P,t}}{\|\mathbf{e}_{y_t}\|\|\boldsymbol{\mu}_{P,t}\|},
\end{equation}
where $\mathbf{e}_{y_t}$ is the embedding vector of token $y_t$. Lower drift means the perturbed token still aligns with the original average belief, while higher drift indicates a stronger rotation after irrelevant features are removed.
We average this drift over $K'$ stochastic sampling passes and their tokens to obtain the sequence-level EED:
\begin{equation}
\label{eq:u_eed}
\mathcal{U}_{\text{EED}}(v, q) = \frac{1}{K'} \sum_{k=1}^{K'} \sum_{t=1}^{T_k}
D_{\text{drift}}(y_t^{(K+k)}),
\end{equation}

This aggregates the step-wise deviations into a scalar score, where $\mathbf{e}_{y_t^{(k)}}$ is the embedding of the token generated by the perturbed model in the $k$-th sample.
Notably, EED scores require far fewer samples than SD to stabilize, thanks to its non-discrete semantic aggregation. It also avoids a semantic clustering model and $\mathcal{O}((K+K')^2)$ entailment judgments, reducing the overhead to just $\mathcal{O}(K')$. We present the framework of the proposed uncertainty metrics in Fig.~\ref{fig:main_figure}

\paragraph{Theoretical Justification}
To justify EED, we show that under a geometric distribution assumption, minimizing $\mathcal{U}_{\text{EED}}$ is equivalent to minimizing the KL divergence between the semantic distributions.
\begin{assumption}[Spherical Semantic Distribution]
We model the continuous semantic distribution of the MLLM output as a von Mises-Fisher (vMF) distribution on the unit hypersphere $\mathbb{S}^{d-1}$~\cite{mardia2009directional,kumarmises}. The probability density of a semantic vector $\mathbf{z}$ given mean direction $\boldsymbol{\mu} \in \mathbb{S}^{d-1}$ and concentration $\kappa > 0$ is:
\begin{equation}
    p(\mathbf{z} | \boldsymbol{\mu}, \kappa) = C_d(\kappa) \exp(\kappa \boldsymbol{\mu}^\top \mathbf{z}),
\end{equation}
where $C_d(\kappa) = \kappa^{d/2-1} / ((2\pi)^{d/2} I_{d/2-1}(\kappa))$ is the normalization constant.
\end{assumption}

This connects the discrete token space to the continuous geometry, where $\boldsymbol{\mu}$ corresponds to the normalized expected embedding $\boldsymbol{\mu}_{P,t}$.

\begin{lemma}[KL Divergence of vMF Distributions]\label{lem:kl_vmf}
Let $Q \sim \text{vMF}(\boldsymbol{\mu}_Q, \kappa)$ and $P \sim \text{vMF}(\boldsymbol{\mu}_P, \kappa)$ be two von Mises-Fisher distributions with identical concentration $\kappa$. The Kullback-Leibler divergence between them is strictly proportional to the cosine distance between their mean directions:
\begin{equation}\label{eq:E_u_eed_is_kl}
    D_{KL}(Q \| P) = \kappa A_d(\kappa) (1 - \boldsymbol{\mu}_Q^\top \boldsymbol{\mu}_P),
\end{equation}
where $A_d(\kappa) = \|\mathbb{E}_{\mathbf{z} \sim Q}[\mathbf{z}]\| = \frac{I_{d/2}(\kappa)}{I_{d/2-1}(\kappa)}$ is the ratio of Bessel functions.
\end{lemma}

The derivation is detailed in Appendix \ref{sec:proof_lemma_vmf}.

\begin{theorem}[Geometric Proxy Equivalence]\label{thm_eed}
The Expected Embedding Drift $\mathcal{U}_{\text{EED}}$ is an unbiased estimator of the KL divergence between the perturbed semantic distribution $Q$ and the original belief $P$, up to a scaling factor and additive constant.
\end{theorem}

The full proof is available in Appendix \ref{sec:proof_theorem_eed}, which connects the proxy to our core objective. The point-wise drift in Eq.~\eqref{eq:drift} is aggregated in Eq.~\eqref{eq:u_eed} to estimate the expected cosine distance. As established in Eq.~\eqref{eq:E_u_eed_is_kl} and Thm.~\ref{thm_eed}, this distance effectively estimates the Semantic Divergence in Eq.~\eqref{eq:divergence}, giving an efficient proxy for epistemic uncertainty.

\section{Experiment}
\label{sec:exp}

We adopt InternVL2-8B~\cite{chen2024far} and Qwen2.5-VL-7B-Instruct~\cite{bai2025qwen25vltechnicalreport} as our evaluated models. We test the performance of $\mathcal U_{\text{sem}}$ and $\mathcal U_{\text{EED}}$ on various benchmarks, including common visual question answering (VQAv2~\cite{goyal2017making}, OKVQA~\cite{marino2019ok}, AdVQA~\cite{li2021adversarial}) and hallucination-targeted benchmarks (POPE~\cite{li2023evaluating}, MME~\cite{fu2025mme}).
For simplicity, we adopt the lightweight GroundingDINO detector~\cite{liu2024grounding} as our RoI generator. In practice, specialized or query-conditioned detectors can also instantiate $T(\cdot)$, as verified in Appendix~\ref{sec:florence_detector}.
We fix the test subset and set an equal maximum sampling time $K_{\text{total}}=K$ (baseline methods) or $K_{\text{total}}=K+K'$ (our method) to $32$ for fair comparison. In practice, as $\{y_k\}_{k=1}^{K}\sim P$ can be utilized to yield our values while producing a half-sampled SE (compared to $K_{\text{total}}$ samples), we calculate a mixed uncertainty metric by interpolating our values with this SE to balance estimation of distinct uncertainty types. 
Since hallucination detection benefits from total uncertainty, this interpolation captures both aleatoric and epistemic uncertainty.
Appendix~\ref{sec:exp_setup_details} provides the detailed experimental setup.

\begin{table*}[!t]
\begin{center}
\caption{\textbf{Uncertainty estimation performance} (\%) on visual question answering benchmarks. AUC is short for AUROC in the table.}
\label{table_vqa_main}
\vskip 0.1in
\begin{center}
\begin{small}
\begin{sc}
\resizebox{0.99\textwidth}{!}{
\begin{tabular}{lcccccccccccc}
\toprule
Benchmark & \multicolumn{4}{c}{VQAv2} & \multicolumn{4}{c}{OKVQA} & \multicolumn{4}{c}{AdVQA} \\
MLLM & \multicolumn{2}{c}{InternVL2-8B} & \multicolumn{2}{c}{Qwen2.5-VL-7B} & \multicolumn{2}{c}{InternVL2-8B} & \multicolumn{2}{c}{Qwen2.5-VL-7B} & \multicolumn{2}{c}{InternVL2-8B} & \multicolumn{2}{c}{Qwen2.5-VL-7B} \\
Metric & AUC & AURAC & AUC & AURAC & AUC & AURAC & AUC & AURAC & AUC & AURAC & AUC & AURAC \\
\midrule
Perplexity & 70.3 & 80.7 & 77.6 & \uline{81.4} & 65.4 & 76.8 & 75.1 & 82.3 & 65.7 & 60.3 & 66.4 & 61.2 \\
Self Evaluation~\cite{xiong2023can} & 60.9 & 75.1 & 62.4 & 71.9 & 66.4 & 77.0 & 63.1 & 75.5 & 60.4 & 57.1 & 55.3 & 49.7 \\
Verbalized Confidence~\cite{tian2023just} & 54.3 & 68.3 & 57.2 & 69.2 & 55.3 & 67.8 & 58.2 & 72.1 & 50.4 & 48.6 & 57.1 & 51.9 \\
\midrule
Semantic Entropy~\cite{farquhar2024detecting} & 78.7 & 82.5 & 76.5 & 79.5 & 80.6 & \uline{84.5} & 78.6 & 83.8 & 75.5 & \uline{66.6} & 73.6 & 63.1 \\
EigenScore~\cite{cheninside} & 72.6 & 80.7 & 75.8 & 79.9 & 72.2 & 78.5 & 76.8 & 83.0 & 67.1 & 60.2 & 64.6 & 59.1 \\
UMPIRE~\cite{lau2025uncertainty} & 74.3 & 81.7 & 72.9 & 78.4 & 72.6 & 79.0 & 76.4 & 83.0 & 68.2 & 61.9 & 65.5 & 59.6 \\
\midrule
Uncertainty-o~\cite{zhang2025uncertainty} & 72.1 & 80.1 & 64.9 & 74.2 & 67.7 & 77.8 & 61.8 & 74.2 & 60.1 & 56.1 & 57.4 & 53.7 \\
Ours (SD) & \textbf{81.3} & \textbf{85.2} & \textbf{78.4} & \textbf{81.5} & \textbf{81.8} & \textbf{85.7} & \textbf{79.6} & \textbf{84.2} & \textbf{76.8} & \textbf{66.9} & \uline{74.5} & \textbf{64.8} \\
Ours   (EED) & \uline{80.3} & \uline{84.5} & \uline{77.9} & 81.1 & \uline{80.9} & \uline{84.5} & \uline{79.3} & \uline{84.0} & \uline{75.8} & 66.5 & \textbf{74.6} & \uline{64.6} \\
\bottomrule
\end{tabular}
 }
\end{sc}
\end{small}
\end{center}
\vskip -0.1in
\end{center}
\end{table*}

\begin{table*}[!t]
\begin{center}
\caption{\textbf{Uncertainty estimation performance} (\%) on all 3 sub-tasks of the POPE benchmark.}
\label{table_pope}
\vskip 0.1in
\begin{center}
\begin{small}
\begin{sc}
\resizebox{0.99\textwidth}{!}{
\begin{tabular}{lcccccccccccc}
\toprule
Benchmark & \multicolumn{4}{c}{POPE random} & \multicolumn{4}{c}{POPE popular} & \multicolumn{4}{c}{POPE adversarial} \\
MLLM & \multicolumn{2}{c}{InternVL} & \multicolumn{2}{c}{QwenVL} & \multicolumn{2}{c}{InternVL} & \multicolumn{2}{c}{QwenVL} & \multicolumn{2}{c}{InternVL} & \multicolumn{2}{c}{QwenVL} \\
Metric & AUC & AURAC & AUC & AURAC & AUC & AURAC & AUC & AURAC & AUC & AURAC & AUC & AURAC \\ \midrule
Perplexity & 68.7 & 91.2 & 68.5 & 91.5 & 68.1 & 90.4 & 67.3 & 90.2 & 68.9 & 90.8 & 65.2 & 89.2 \\
Self Evaluation~\cite{xiong2023can} & 52.0 & 84.6 & 61.3 & 91.4 & 53.6 & 84.7 & 51.5 & 88.1 & 53.5 & 85.2 & 57.4 & 89.5 \\
Verbalized Confidence~\cite{tian2023just} & 50.9 & 84.2 & 60.4 & 88.2 & 51.8 & 82.9 & 60.5 & 87.2 & 50.6 & 82.8 & 60.6 & 87.5 \\ \midrule
Semantic Entropy~\cite{farquhar2024detecting} & 70.6 & 89.6 & 62.9 & 88.7 & 68.7 & 88.5 & 64.8 & 88.0 & 72.3 & 90.2 & 63.0 & 86.3 \\
EigenScore~\cite{cheninside} & 65.5 & 89.4 & 56.0 & 87.2 & 64.4 & 87.8 & 56.6 & 86.5 & 63.8 & 88.5 & 54.6 & 85.4 \\
UMPIRE~\cite{lau2025uncertainty} & 68.7 & 91.2 & \uline{80.5} & \uline{94.4} & 68.7 & \uline{90.5} & \uline{79.1} & \uline{93.4} & 68.2 & 90.5 & \uline{77.0} & \uline{92.7} \\ \midrule
Uncertainty-o~\cite{zhang2025uncertainty} & 64.6 & 88.8 & 51.5 & 85.4 & 66.4 & 88.6 & 55.1 & 85.9 & 62.8 & 88.2 & 56.6 & 86.1 \\
Ours (SD) & \textbf{75.9} & \textbf{92.4} & 66.2 & 89.3 & \textbf{74.8} & \textbf{91.3} & 66.1 & 88.7 & \textbf{74.9} & \textbf{91.7} & 64.5 & 88.1 \\
Ours (EED) & \uline{73.1} & \uline{91.6} & \textbf{83.2} & \textbf{95.1} & \uline{71.0} & \uline{90.5} & \textbf{81.5} & \textbf{94.3} & \uline{73.0} & \uline{91.3} & \textbf{78.6} & \textbf{93.6} \\ \bottomrule
\end{tabular}
}
\end{sc}
\end{small}
\end{center}
\vskip -0.1in
\end{center}
\end{table*}

\subsection{Results on Hallucination Detection}
\label{sec:main_results}

\paragraph{General VQA Benchmarks.}
Table \ref{table_vqa_main} presents the results on VQAv2, OKVQA, and AdVQA. CIM consistently outperforms existing baselines all evaluated VQA benchmarks.
Notably, on the AdVQA benchmark, which is specifically designed to elicit model errors, our method achieves an AUROC of \textbf{76.8\%} on InternVL, surpassing Semantic Entropy (75.5\%) and EigenScore (67.1\%). This suggests that our method's specific reliance on \textit{causal invariance} by checking if the answer holds after removing background context is more robust to adversarial visual perturbations than methods that solely attend to the output distribution.
Furthermore, our computationally efficient proxy, Ours (EED), performs competitively with the $\mathcal{U}_{\text{sem}}$ method, often within 1\% of the explicit Semantic Divergence scores, validating the theoretical connection in Lemma \ref{lem:kl_vmf}.

We also provide case study in Appendix~\ref{sec:case_study}, to further verify that SD more successfully capture epistemic uncertainty.

\paragraph{Hallucination-targeted Benchmarks.}
Table \ref{table_pope} details performance on the POPE benchmark. This benchmark is crucial as it directly evaluates on the topic of object hallucinations.
On the \textbf{InternVL} model, our explicit method dominates the performance across all splits, achieving $\sim$\textbf{75\% AUROC}, significantly outperforming Semantic Entropy ($\sim$68-72\%).
Interestingly, on \textbf{QwenVL}, our EED proxy outperforms the explicit method (e.g., 83.2\% vs 66.2\% on Random). We find that QwenVL's decoding distribution is much sharper than InternVL's on POPE (Appendix~\ref{sec:pope_entropy}), which makes discrete semantic clustering less stable for the explicit method. In contrast, EED operates in the continuous embedding space and is less affected by hard cluster assignments.
A significant gap is observed between our CIM and other baselines on the challenging \textbf{Adversarial} split, confirming CIM's identification of whether the context priors are relied on rather than visual evidence. We further evaluate on the comprehensive MME benchmark in Appendix~\ref{sec:mme_results}, where Fig.~\ref{fig:MME_performance} shows that CIM remains the best-performing method.
This demonstrates CIM's generalization beyond simple object existence to complex reasoning in MME, and again verifies its consistency in uncertainty estimation.

\subsection{Discussions}
\paragraph{Causal Masking.}
Our theoretical framework proposes that the visual input operator $T(\cdot)$ must focus on isolating causal features $x_c$ to separate epistemic from aleatoric uncertainty. To validate such a design choice of $x_c$-focused RoI Cropping, we compare it against: 1) \textbf{Random Crop}, which crops an arbitrary region, 2) \textbf{Reversed Crop}, which masks the detected RoIs and keeps only the background, and 3) \textbf{Ground-Truth RoIs}, which are from the box-annotated GQA benchmark.
The ablation results are presented in Table~\ref{table_vp_operation_ablation}. We also compare with the commonly used gaussian noise perturbation in Appendix~\ref{sec:gaussian_noise_ablation}.
From the results, \textbf{Random Crop} yield inferior performance compared to our targeted approach. This aligns with our hypothesis that random perturbations degrade the aleatoric and epistemic signals indiscriminately.
Crucially, \textbf{Reversed Crop} performs significantly worse than Random Crop in most tested cases. This is because feeding only the background forces the model to hallucinate or guess, but does not provide an anchor point for the correct reasoning.
In contrast, our method of RoI Crop achieves significantly larger AUROC across all settings (e.g., \textbf{81.3\%} on VQAv2 for InternVL). By preserving necessary evidence while maximally removing spurious context, our measurement effectively grounds the MLLM's confidence in \textit{causal} signals, providing the cleanest isolation of epistemic uncertainty.

While the SE component is included to retain aleatoric sensitivity for hallucination detection, it can dilute the random-vs-RoI gap on epistemic UQ.
To isolate the epistemic component, the $\text{\scshape AUROC}^{\text{\scshape epi}}$ metric reports AUROCs in a pure SD setup without the auxiliary SE score. Under this setting, the separation becomes much larger: Our RoI Crop improves over Random Crop by 6.5/9.2 AUROC on InternVL VQAv2/AdVQA or by 11.6/14.0 AUROC on QwenVL VQAv2/AdVQA, and is even more superior to the Reversed Crop.
We additionally include GQA~\cite{hudson2019gqa} because it provides ground-truth object annotations, enabling a controlled comparison with GT RoIs. On GQA, our operation is comparable to \textbf{GT RoIs Crop} and both clearly outperform the others, validating that CIM relies on preserving query-relevant causal evidence rather than arbitrary perturbation, and exact crops are less critical than retaining causal evidence.

\begin{table*}[!t]
\centering
\caption{\textbf{Ablation study of the $x_c$-focused image transform.} AUROC and AURAC use the final mixed score, while $\text{\scshape AUROC}^{\text{\scshape epi}}$ uses the SE-free pure SD score. GT RoI is available only on GQA; -- denotes not applicable. IVL / QVL denotes InternVL / QwenVL, respectively.}
\label{table_vp_operation_ablation}
\vskip 0.1in
\begin{small}
\begin{sc}
\setlength{\tabcolsep}{2.0pt}
\resizebox{0.99\textwidth}{!}{
\begin{tabular}{lcccccccccccccccccc}
\toprule
Metric & \multicolumn{6}{c}{$\text{AUROC}$} & \multicolumn{6}{c}{$\text{AUROC}^{\text{epi}}$} & \multicolumn{6}{c}{$\text{AURAC}$} \\
\cmidrule(lr){2-7}\cmidrule(lr){8-13}\cmidrule(lr){14-19}
Benchmark & \multicolumn{2}{c}{VQAv2} & \multicolumn{2}{c}{AdVQA} & \multicolumn{2}{c}{GQA} & \multicolumn{2}{c}{VQAv2} & \multicolumn{2}{c}{AdVQA} & \multicolumn{2}{c}{GQA} & \multicolumn{2}{c}{VQAv2} & \multicolumn{2}{c}{AdVQA} & \multicolumn{2}{c}{GQA} \\
MLLM & IVL & QVL & IVL & QVL & IVL & QVL & IVL & QVL & IVL & QVL & IVL & QVL & IVL & QVL & IVL & QVL & IVL & QVL \\
\midrule
Random Crop     & \uline{80.2}  & \uline{77.9}  & \uline{75.7}  & \uline{71.9}  & 77.9          & 68.8          & \uline{70.3}  & \uline{62.1}  & \uline{61.4}  & \uline{54.3}  & 61.0          & 57.5          & \uline{83.6}  & \uline{81.6}  & \uline{66.7}  & \uline{61.5}  & 81.5          & 61.3          \\
Reversed Crop   & 77.6          & 76.1          & 74.8          & \uline{71.9}  & 76.8          & 69.5          & 53.7          & 56.5          & 54.2          & 50.7          & 58.7          & 54.5          & 81.5          & 79.2          & 65.0          & 61.3          & 80.1          & 61.5          \\
GT RoI Crop     & --            & --            & --            & --            & \textbf{80.3} & \textbf{71.4} & --            & --            & --            & --            & \uline{72.1}  & \textbf{66.3} & --            & --            & --            & --            & \uline{82.4}  & \textbf{63.9} \\
Ours (RoI Crop) & \textbf{81.3} & \textbf{78.5} & \textbf{76.8} & \textbf{74.5} & \uline{79.9}  & \uline{70.2}  & \textbf{76.8} & \textbf{73.7} & \textbf{70.6} & \textbf{68.3} & \textbf{73.3} & \uline{64.6}  & \textbf{85.2} & \textbf{81.9} & \textbf{66.9} & \textbf{64.8} & \textbf{82.5} & \uline{63.1}  \\
\bottomrule
\end{tabular}
}
\end{sc}
\end{small}
\vskip -0.1in
\end{table*}

\begin{figure}[t]
\centering
\begin{minipage}[t]{0.49\textwidth}
\vspace{0pt}
\centering
\includegraphics[width=0.49\linewidth]{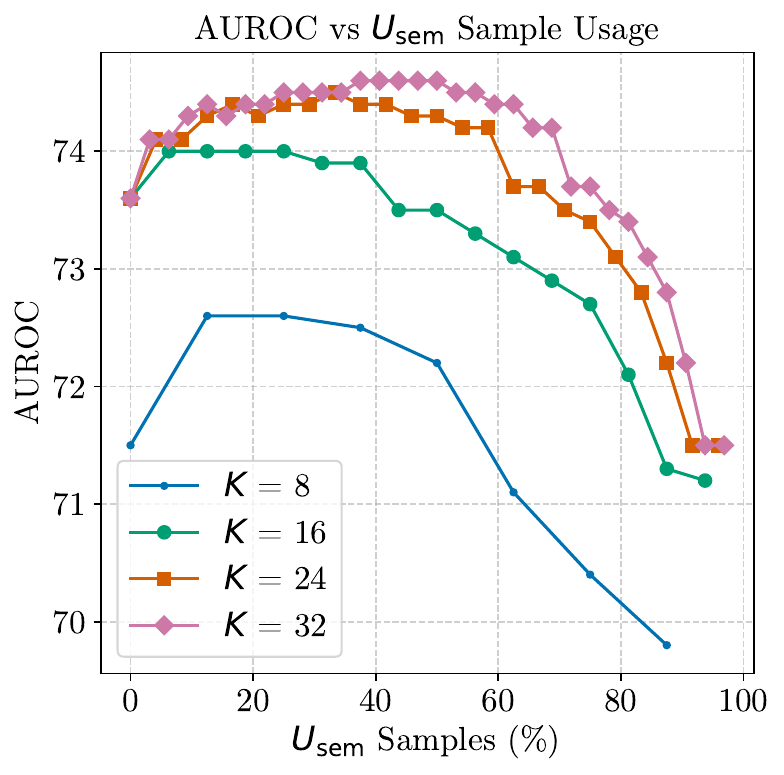}
\includegraphics[width=0.49\linewidth]{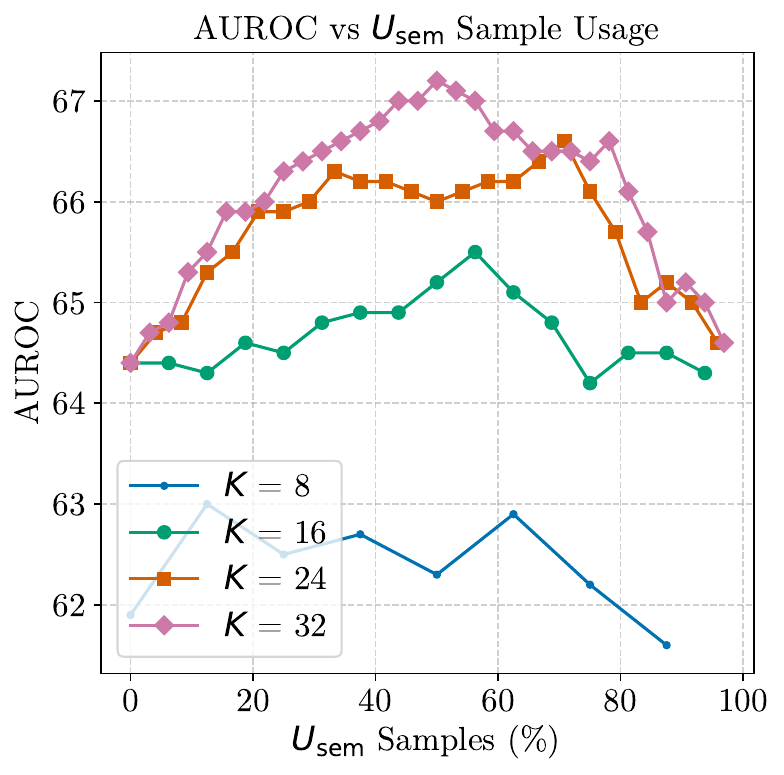}
\caption{\textbf{Hallucination detection performance with fixed numbers of total samples}, \textit{w.r.t.} the proportion of samples for $\mathcal{U}_\text{sem}$ calculation besides SE calculation, showing QwenVL results using AdVQA (left) and MME (right).}
\label{fig:fixed_total_samples}
\end{minipage}
\hfill
\begin{minipage}[t]{0.48\textwidth}
\vspace{0pt}
\centering
\includegraphics[width=0.49\linewidth]{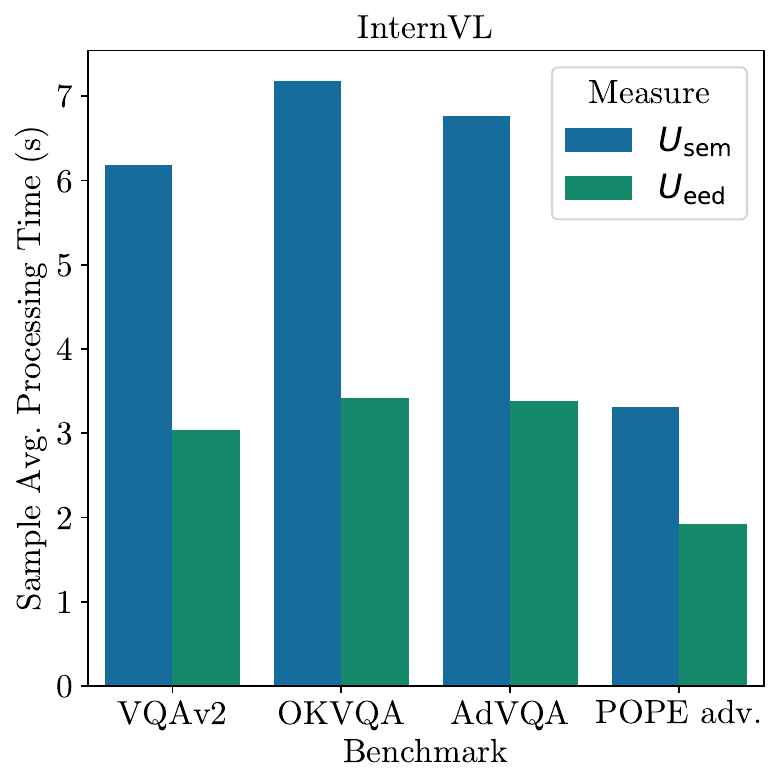}
\includegraphics[width=0.49\linewidth]{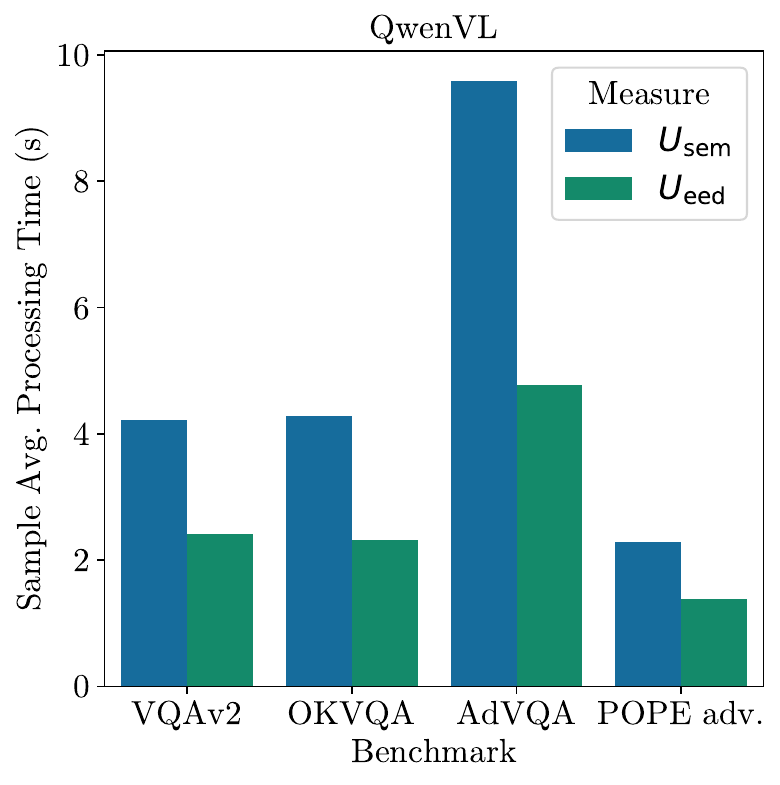}
\caption{\textbf{Sample-wise efficiency comparison for $\mathcal{U}_\text{sem}$ and $\mathcal{U}_\text{EED}$.} It takes only half time for $\mathcal{U}_\text{EED}$ to compute the uncertainty under saturated performance.}
\label{fig:sem_eed_efficiency_compare}
\end{minipage}
\vskip -0.1in
\end{figure}

\paragraph{Number of Samples.}
In Fig. \ref{fig:fixed_total_samples}, we show the optimal sampling allocation under a fixed total number $K + K'$ between estimating the original $P(s|v)$, \textit{i.e.}, for SE and $\mathcal U_{\text{sem}}$, and measuring the masked $P(s|v')$ only for $\mathcal U_{\text{sem}}$. We observe a consistent middle-peaking trend across benchmarks. The maximal performance is achieved at a balanced sampling budget of 40\%-60\% each. Biased allocations degrade performance, as insufficient samples for either distribution lead to mis-estimation of one type of uncertainty. Furthermore, while increasing the total budget universally improves AUROC, the optimal allocation proportion remains stable, suggesting its robustness in such scaling.

\begin{wrapfigure}{r}{0.40\textwidth}
\centering
\includegraphics[width=\linewidth]{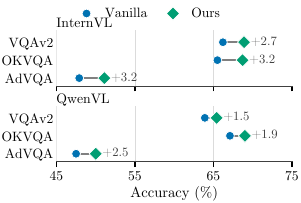}
\caption{\textbf{Hallucination mitigation performance} (Acc\%) using uncertainty values estimated by our proposed method to threshold and re-answer the query.}
\label{fig:hallucination_mitigation}
\vspace{-0.1in}
\end{wrapfigure}

\paragraph{Efficiency of EED.}
Fig. \ref{fig:sem_eed_efficiency_compare} compares the average processing time per sample between the Semantic Divergence $\mathcal U_{\text{sem}}$ and our proposed proxy $\mathcal U_{\text{EED}}$. Across both MLLM architectures and all four benchmarks, $\mathcal U_{\text{EED}}$ consistently reduces the computational cost by approximately 50\%. For instance, on the computationally demanding AdVQA benchmark with QwenVL, the processing time drops from $\sim9.6$s to $\sim4.8$s per sample.
Compared with sampling-and-clustering methods under the same generation budget, EED removes the NLI-based $\mathcal O((K+K')^2)$-time clustering for the perturbed input and operates directly in the embedding space. Appendix~\ref{sec:nli_ablation} further shows that SD is not strongly tied to one entailment model, but EED avoids that clustering dependency altogether.

\begin{wrapfigure}{r}{0.64\textwidth}
\centering
\includegraphics[width=0.49\linewidth]{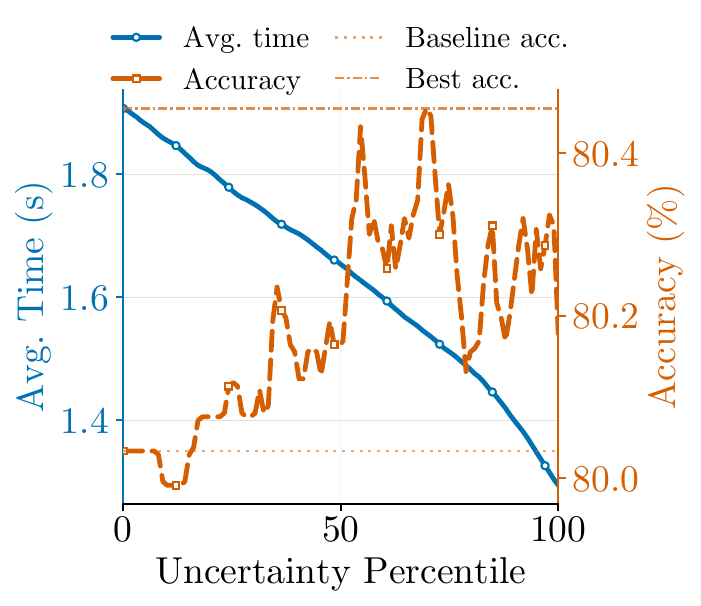}
\includegraphics[width=0.49\linewidth]{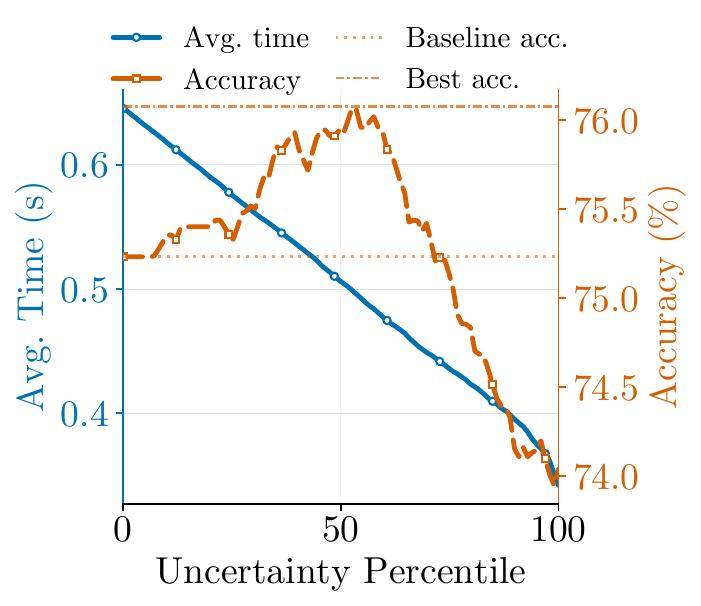}
\caption{\textbf{Adaptive-resolution VQA performance} on MME using different uncertainty thresholds as the indicator: confident samples are processed under low resolutions. Left: w/ InternVL; Right: w/ QwenVL.}
\label{fig:adaptive_resolution_boost}
\vspace{-0.1in}
\end{wrapfigure}

\paragraph{Applications.}
Beyond detection, a practical uncertainty metric should enable the model to correct its own errors. We study whether \textbf{CIM} can act as a filter to improve factual quality. In this setting, we perform a reject-and-refine strategy: if the estimated uncertainty $\mathcal{U}_{\text{sem}}$ exceeds a threshold, we instruct the model to re-generate the answer. Accuracy improvements on VQAv2, OKVQA, and AdVQA are reported in Fig. \ref{fig:hallucination_mitigation}, which demonstrates the effectiveness of such intervention. By using our uncertainty estimates to trigger re-answering, we observe consistent accuracy gains across all benchmarks.
Importantly, on the difficult \textbf{AdVQA} benchmark, the performance of InternVL improves from 47.9\% to \textbf{51.1\%}, and QwenVL from 47.5\% to \textbf{50.0\%}. This significant increase indicates that our method successfully flags epistemic failures where the model was previously hallucinating due to adversarial visual contexts.

We further apply CIM to adaptive-resolution inference, where samples with uncertainty below a percentile threshold are processed at a low maximum resolution of $256$, reserving high-resolution inference for uncertain samples.
As shown in Fig.~\ref{fig:adaptive_resolution_boost}, thresholding confident samples reduces inference time and also boosts accuracy beyond the high-resolution baseline.
For QwenVL, downscaling the most confident 50\% samples reduces average latency from 0.65s to 0.50s and raises accuracy from $\sim75.2\%$ to $\sim76.1\%$; InternVL shows similar gains around the 50th--70th percentile thresholds.
This suggests that high resolution can distract the inference on already sufficient samples, while CIM identifies when low-resolution evidence is adequate.

\section{Conclusion}
\label{sec:conclusion}

This work addresses the critical bias in existing MLLM uncertainty quantification methods toward aleatoric uncertainty, which often ignores hallucinations driven by spurious correlations. We proposed CIM, a framework that isolates epistemic uncertainty by measuring the semantic shift under causally invariant perturbations with theoretical guarantee. Our proposed Semantic Divergence effectively targets reliance on non-causal features, achieving state-of-the-art performance. Furthermore, the Expected Embedding Drift (EED) reduces computational latency by 50\% without compromising accuracy. By enabling robust hallucination detection and efficient adaptive inference, CIM enables a rigorous path toward trustworthy multimodal reasoning.

\paragraph{Limitations.}
CIM depends on the quality of the RoIs and equivalence estimation, so failures of them can affect the results. Nevertheless, we empirically show their impact is limited.

\begin{ack}
Use unnumbered first level headings for the acknowledgments. All acknowledgments
go at the end of the paper before the list of references. Moreover, you are required to declare
funding (financial activities supporting the submitted work) and competing interests (related financial activities outside the submitted work).
More information about this disclosure can be found at: \url{https://neurips.cc/Conferences/2026/PaperInformation/FundingDisclosure}.

Do {\bf not} include this section in the anonymized submission, only in the final paper. You can use the \texttt{ack} environment provided in the style file to automatically hide this section in the anonymized submission.
\end{ack}

\bibliographystyle{plainnat}
\bibliography{references}

\newpage
\appendix

{\Large \textbf{Appendix}\vspace{1.0ex}}
\section{Proofs of Equations and Propositions}

\subsection{Proof of Lemma~\ref{lemma:alertoric} (Aleatoric Cancellation)}\label{sec:proof_lemma_aleatoric}
\begin{proof}
Let $P(s) \triangleq P_\theta(s|v)$ and $Q(s) \triangleq P_\theta(s|v')$ be the distribution given the original input and the masked input $v' = T(v)$, respectively.
The log-probabilities for these inputs are
\begin{align*}
    \log Q(s) &= -(E_c(s, v_c) + E_s(s, v'_s)) - \log Z(v') \\
    \log P(s) &= -(E_c(s, v_c) + E_s(s, v_s)) - \log Z(v).
\end{align*}
Subtracting the two equations yields:
\begin{align} \label{eq:lemma3.3}
    \log \frac{Q(s)}{P(s)} &= (E_s(s, v_s) - E_s(s, v'_s)) + \log \frac{Z(v)}{Z(v')} \nonumber \\
    &= \Delta_s(s) + \log \frac{Z(v)}{Z(v')}.
\end{align}
\end{proof}

\subsection{Proof of Theorem~\ref{theorem:jensen} (Epistemic Jensen Gap)}\label{sec:proof_theorem_jensen}
\textbf{Part 1: Equivalence to Jensen Gap}
\begin{proof}
We first expand the partition function ratio. By definition, $Z(v) = \sum_y \exp(-(E_c(s, v_c) + E_s(s, v_s)))$. The exponent can be rewritten as $-E_c - E_s(v'_s) - (E_s(v_s) - E_s(v'_s)) = -E(s, v') - \Delta_s(s)$. Therefore, we can derive that
\begin{equation}\label{zv_derive_1}
Z(v) = \sum_{s} e^{-E_\theta(s, v')} e^{-\Delta_s(s)}.
\end{equation}
By definition of $Z(v')$, we have $e^{-E_\theta(s, v')}=Z(v')Q(s)$ which simplifies Eq.~\eqref{zv_derive_1} as
\begin{align}
Z(v) &= \sum_{s} Z(v')Q(s) e^{-\Delta_s(s)}\nonumber\\
&= Z(v') \cdot \mathbb{E}_{s \sim Q}[e^{-\Delta_s(s)}]
 \end{align}
Thus, the constant term from Lemma \ref{lemma:alertoric} is $\log \frac{Z(v)}{Z(v')} = \log \mathbb{E}_{y \sim Q}[e^{-\Delta_s(s)}]$.
Now, substituting Lemma \ref{lemma:alertoric} into the definition in Eq.~\eqref{eq:divergence} yields
\begin{align}
\label{eq:jensen_gap}
    \mathcal{U}_{\text{sem}} &= \mathbb{E}_{s \sim Q} \left[ \log \frac{Q(s)}{P(y)} \right] \nonumber \\
    &= \mathbb{E}_{s \sim Q} \left[ \Delta_s(s) + \log \mathbb{E}_{s' \sim Q}[e^{-\Delta_s(s')}] \right] \nonumber \\
    &= \log \mathbb{E}_{s \sim Q}[e^{-\Delta_s(s)}] - \mathbb{E}_{s \sim Q}[-\Delta_s(s)],
\end{align}
which is precisely the Jensen Gap $\mathcal{J}(e^x; -\Delta_s) \ge 0$.
The divergence is strictly zero if and only if $\Delta_s(s)$ is constant almost surely over $\mathcal{S}$, implying the model is insensitive to spurious features.
While Eq.~\eqref{eq:jensen_gap} provides an exact measure, it can be interpreted best at a low hallucination scenario, \textit{i.e.}, the model is trained well enough to produce relatively smaller $|\Delta_s(s)|$, which results in Eq.~\ref{eq:u_sem_prop_to_var}.
\end{proof}

\textbf{Part 2: Convergence to Variance (Proof of Eq.~\eqref{eq:u_sem_prop_to_var})}
\begin{proof}
Let $K_{-\Delta}(t) = \log \mathbb{E}_{s \sim Q}[e^{t(-\Delta_s(s))}]$ be the Cumulant Generating Function (CGF) of the random variable $-\Delta_s(s)$. From Theorem 1, we identify that:
\begin{equation}
    D_{KL}(Q \| P) = K_{-\Delta}(1) - \mathbb{E}_{s \sim Q}[-\Delta_s(s)] = K_{-\Delta}(1) + \mathbb{E}_{s \sim Q}[\Delta_s(s)]
\end{equation}
We apply Taylor's Theorem with the Lagrange form of the remainder to $K_{-\Delta}(t)$ on the interval $[0,1]$. There exists some $\xi \in (0,1)$ such that:
\begin{equation}
    K_{-\Delta}(1) = K_{-\Delta}(0) + K'_{-\Delta}(0)(1) + \frac{K''_{-\Delta}(0)}{2!}(1)^2 + \frac{K'''_{-\Delta}(\xi)}{3!}(1)^3
\end{equation}
Using the properties of cumulants, we have $K_{-\Delta}(0) = 0$, $K'_{-\Delta}(0) = \mathbb{E}_Q[-\Delta_s(s)]$, and $K''_{-\Delta}(0) = \mathrm{Var}_Q(-\Delta_s(s))$.
Given the assumption that the shift is bounded by $\sup |\Delta_s(s)| \le \epsilon$, the third derivative satisfies $|K^{(3)}_{-\Delta}(\xi)| \le C\epsilon^3$ for some constant $C$, effectively bounding the remainder term by $\mathcal{O}(\epsilon^3)$.
Substituting these values:
\begin{equation}
    K_{-\Delta}(1) = -\mathbb{E}_Q[\Delta_s(s)] + \frac{1}{2}\mathrm{Var}_Q(\Delta_s(s)) + \mathcal{O}(\epsilon^3)
\end{equation}
Substituting this expansion back into the KL equation, the linear expectation terms cancel:
\begin{align}
    D_{KL}(Q \| P) &= \left( -\mathbb{E}_Q[\Delta_s(s)] + \frac{1}{2}\mathrm{Var}_Q(\Delta_s(s)) + \mathcal{O}(\epsilon^3) \right) + \mathbb{E}_Q[\Delta_s(s)] \nonumber \\
    &= \frac{1}{2}\mathrm{Var}_{s \sim Q}[\Delta_s(s)] + \mathcal{O}(\epsilon^3)
\end{align}
\end{proof}

\subsection{Proof of Lemma~\ref{lem:kl_vmf} (KL Divergence of vMF Distributions)}\label{sec:proof_lemma_vmf}
\begin{proof}
By definition, the KL divergence is:
\begin{align}
    D_{KL}(Q \| P) &= \int_{\mathbb{S}^{d-1}} q(\mathbf{z}) \log \frac{q(\mathbf{z})}{p(\mathbf{z})} d\mathbf{z} \nonumber \\
    &= \int_{\mathbb{S}^{d-1}} q(\mathbf{z}) \left( \log \frac{C_d(\kappa) e^{\kappa \boldsymbol{\mu}_Q^\top \mathbf{z}}}{C_d(\kappa) e^{\kappa \boldsymbol{\mu}_P^\top \mathbf{z}}} \right) d\mathbf{z} \nonumber \\
    &= \int_{\mathbb{S}^{d-1}} q(\mathbf{z}) \left( \kappa \boldsymbol{\mu}_Q^\top \mathbf{z} - \kappa \boldsymbol{\mu}_P^\top \mathbf{z} \right) d\mathbf{z} \nonumber \\
    &= \kappa (\boldsymbol{\mu}_Q - \boldsymbol{\mu}_P)^\top \int_{\mathbb{S}^{d-1}} \mathbf{z} q(\mathbf{z}) d\mathbf{z}.
\end{align}
A property of the vMF distribution is that the expected vector is $\mathbb{E}_{\mathbf{z} \sim Q}[\mathbf{z}] = A_d(\kappa) \boldsymbol{\mu}_Q$. Substituting this into the integral:
\begin{align}
    D_{KL}(Q \| P) &= \kappa (\boldsymbol{\mu}_Q - \boldsymbol{\mu}_P)^\top (A_d(\kappa) \boldsymbol{\mu}_Q) \nonumber \\
    &= \kappa A_d(\kappa) (\boldsymbol{\mu}_Q^\top \boldsymbol{\mu}_Q - \boldsymbol{\mu}_P^\top \boldsymbol{\mu}_Q).
\end{align}
Since $\boldsymbol{\mu}_Q$ is a unit vector, $\boldsymbol{\mu}_Q^\top \boldsymbol{\mu}_Q = 1$. Thus, $D_{KL}(Q \| P) = \kappa A_d(\kappa) (1 - \boldsymbol{\mu}_P^\top \boldsymbol{\mu}_Q)$.
\end{proof}

\subsection{Proof of Theorem~\ref{thm_eed} (Geometric Proxy Equivalence)}\label{sec:proof_theorem_eed}
\begin{proof}
The EED metric computes the expectation of the cosine distance over the sampled trajectory from $Q$. Taking the expectation of Eq.~\eqref{eq:u_eed} with respect to $y_t \sim Q$:
\begin{align}
    \mathbb{E}_{y_t \sim Q}[\mathcal{U}_{\text{EED}}] &= 1 - \mathbb{E}_{y_t \sim Q}[\mathbf{e}_{y_t}]^\top \boldsymbol{\mu}_{P} \nonumber\\
    &= 1 - (A_d(\kappa)\boldsymbol{\mu}_Q)^\top \boldsymbol{\mu}_P.
\end{align}
Rearranging terms for $\boldsymbol{\mu}_Q^\top \boldsymbol{\mu}_P$ and substituting into Lemma \ref{lem:kl_vmf}:
\begin{align}
    D_{KL}(Q \| P) &= \kappa A_d(\kappa) \left( 1 - \frac{1 - \mathbb{E}[\mathcal{U}_{\text{EED}}]}{A_d(\kappa)} \right) \nonumber\\
    &= \kappa \mathbb{E}_{y_t \sim Q}[\mathcal{U}_{\text{EED}}] - \kappa (1 - A_d(\kappa)).
\end{align}
The divergence relates to EED via an affine transformation. The term $\kappa (1 - A_d(\kappa))$ is a constant dependent only on concentration $\kappa$ and dimension $d$. Under high concentration ($\kappa \to \infty$), $A_d(\kappa) \to 1$, and the relation simplifies to exact linearity. Thus, minimizing EED strictly minimizes the divergence between the semantic distributions.
\end{proof}

\section{Detailed Experiment Setup}
\label{sec:exp_setup_details}

\paragraph{Computational Resources}
We carried out all experiments on 4 NVIDIA RTX A6000 GPUs. 

\paragraph{Model Parameters}
The implementation uses an inference temperature of 0.1, while the sampling temperature is 0.5. The RoI transform uses GroundingDINO with detection threshold 0.3, which filters boxes labeled as the whole image when they cover more than 95\% of the image area and uses a central crop fallback when no valid detection remains.

\paragraph{Prompts}
The dataset question is passed to the evaluated MLLM with this prompt:
\begin{verbatim}
[Question] Answer this question shortly. 
\end{verbatim}
The prompt used to extract key objects for grounding is:
\begin{verbatim}
What are the key objects in this question-answer pair?
Q: [Question]
A: [Answer]
Output in json format: {"key_objects": ["key object 1", "key object 2", ...]}
\end{verbatim}
where the placeholders are filled with the original question and the model's original answer, and the judge LLM is queried with temperature $0.1$. The prompt used for pairwise semantic clustering of sampled answers is:
\begin{verbatim}
Question: [Question] True Answer: [Answer A] Is the following answer correct?
Answer: [Answer B] Respond with 'no' if this is totally different from the
True Answer. Respond with 'yes' when they are of the same meaning.
\end{verbatim}
This clustering prompt is filled with the text query and the answer pairs, using a temperature of $0.1$.

\paragraph{Causal-Invariant Transform.}
We use GroundingDINO as a lightweight RoI generator for the causal-invariant transform $T(\cdot)$. When no valid RoI is found, we use a central crop as the fallback view. This choice is not essential to the framework: task-specific or query-conditioned detectors can be substituted when they are available.

\paragraph{Dataset Configuration.}
To keep evaluation consistent and computationally feasible, we fix the evaluation subsets across all methods. For hallucination-targeted benchmarks, we use the full test sets: 3,000 POPE examples across the random, popular, and adversarial splits, and 2,374 MME examples. For VQAv2, OKVQA, AdVQA, and GQA, we use the first 2,000 validation examples from each dataset.

\paragraph{Sampling Budget and Fair Comparison.}
For all sampling-based methods, we fix the total decoding budget to $K_{\text{total}}=32$ sequences per query.
\begin{itemize}
    \item \textbf{Baselines:} SE, EigenScore, and UMPIRE sample only from the original distribution $P(y|v)$, so we set $K=32$.
    \item \textbf{Ours (CIM):} CIM samples from both the original distribution $P(y|v)$ and the causally-filtered distribution $P(y|v')$. We use $K=16$ original-view samples and $K'=16$ filtered-view samples, keeping $K+K'=32$.
\end{itemize}

\paragraph{Metric Interpolation.}
The $K=16$ original-view samples used by CIM can also estimate standard Semantic Entropy without extra decoding. We therefore report a mixed score
$\mathcal{U}_{\text{final}} = \lambda\mathcal{U}_{\text{ours},K'=16} + (1 - \lambda) \cdot \mathcal{U}_{\text{SE},K=16}$,
where $\mathcal{U}_{\text{ours}}$ is either $\mathcal{U}_{\text{sem}}$ or $\mathcal{U}_{\text{EED}}$. Here, $\lambda$ is set to $0.05$. This score retains sensitivity to aleatoric uncertainty while adding the epistemic signal measured by CIM.

\section{More Experiments and Results}
\label{sec:more_exp_results}
\subsection{Additional MME Results}
\label{sec:mme_results}

\begin{figure}[t]
\vskip 0.2in
\begin{center}
\centerline{
\hfill
\includegraphics[width=0.6\textwidth]{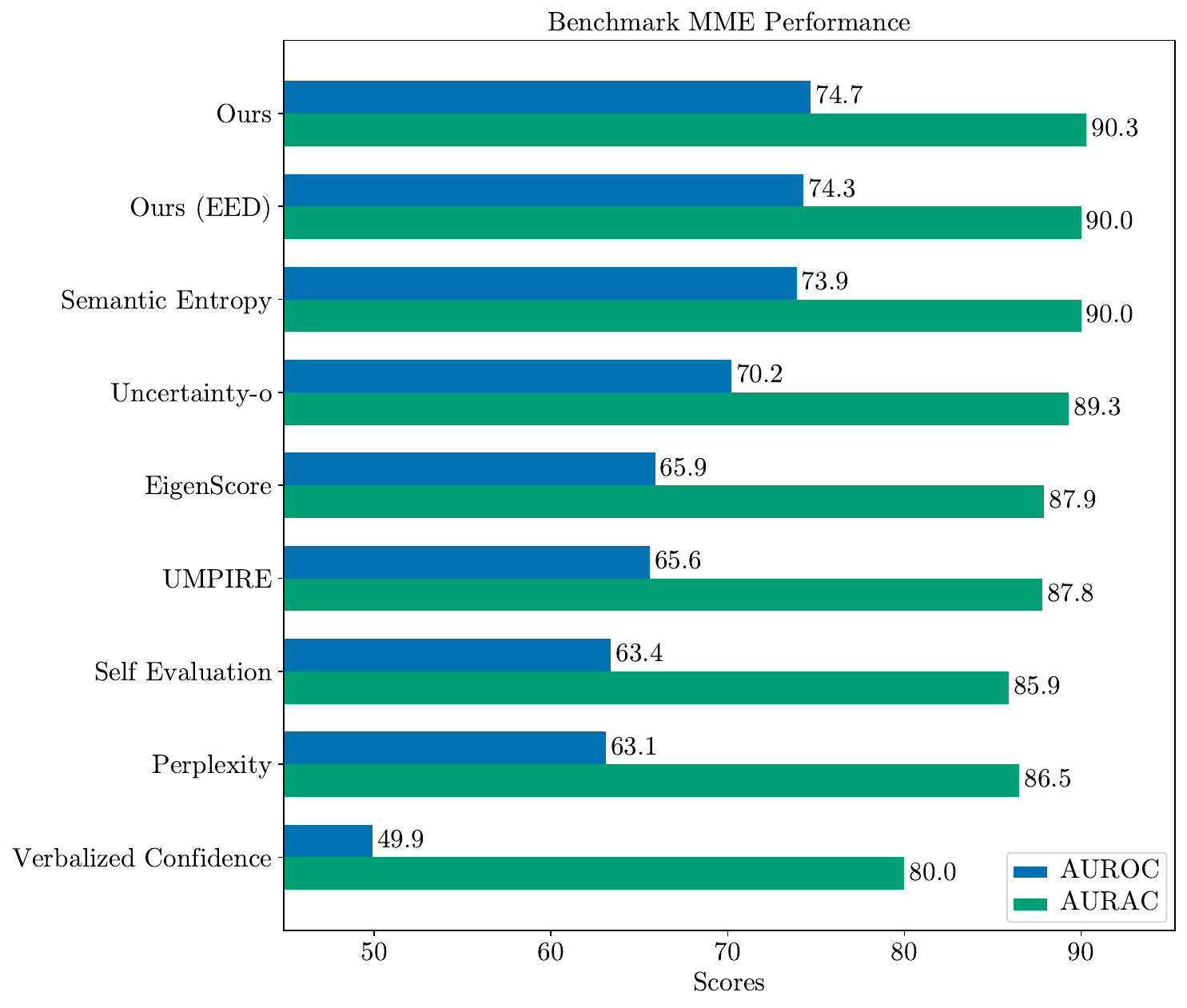}
\hfill
}
\caption{\textbf{Uncertainty estimation performance} (\%) on the MME benchmark.}
\label{fig:MME_performance}
\end{center}
\vskip -0.2in
\end{figure}

Figure~\ref{fig:MME_performance} reports hallucination detection on MME. CIM achieves the best performance (74.7\%), showing that the uncertainty signal generalizes beyond the VQA datasets in the main tables.

\subsection{Case Study}
\label{sec:case_study}

\begin{figure*}[!h]
\centering
\includegraphics[width=0.99\textwidth]{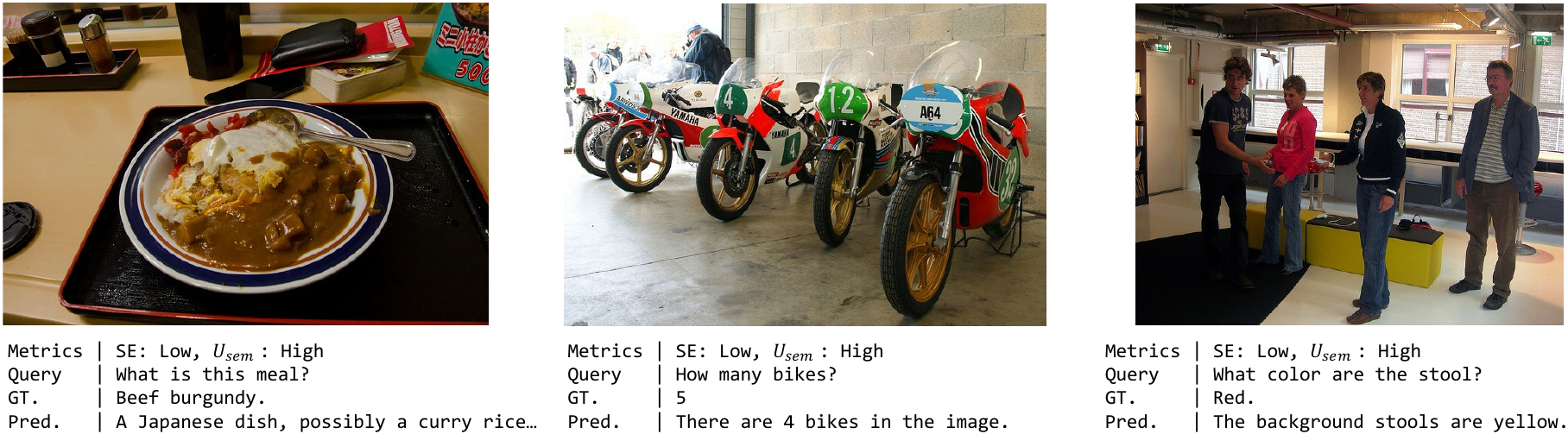}
\caption{\textbf{Representative instances of SE vs. $\mathcal{U}_{\text{sem}}$ on hallucinated samples.} In all instances, our proposed CIM uncertainty is high despite the standard SE remains low, indicating epistemic uncertainty that common sampling-based uncertainty metric fails to identify.}
\label{fig:qualitative_cases}
\end{figure*}

To better understand the behavior of CIM in practice, we present a case analysis of SE's failure cases in Fig. \ref{fig:qualitative_cases}. As illustrated, the MLLM frequently generates incorrect answers with high confidence, leading to extremely low SE scores. For instance, in the left example, the model consistently mispredicts `Japanese dish,' likely misled by background Japanese text. Similarly, in the right example, the model hallucinates the color possibly by looking at an incorrect object. Empirically, counting complex objects can often hallucinate but result in a stable number or low semantic entropy, as shown in the middle instance. In these scenarios, standard metrics fail but the proposed method consistently peaks. Such behavior confirms the successful detection of epistemic failure with our CIM.

\subsection{Gaussian Noise Perturbation}
\label{sec:gaussian_noise_ablation}

We compare RoI masking with the commonly used gaussian noise perturbation under the same VQAv2 and AdVQA settings as the main masking ablation. Table~\ref{tab:gaussian_noise_ablation} reports the final mixed uncertainty score. Gaussian noise perturbs the entire image and therefore cannot selectively preserve $x_c$ while suppressing $x_s$.
Its performance stays close to random masking and below RoI masking across benchmarks. This supports the main conclusion that the improvement is not caused by simply changing the image, but by constructing an evidence-preserving view that removes non-causal context.

\begin{table}[!h]
\begin{center}
\caption{\textbf{Gaussian noise perturbation ablation.} Results are AUROC/AURAC on VQAv2 and AdVQA.}
\label{tab:gaussian_noise_ablation}
\vskip 0.1in
\begin{small}
\begin{sc}
\resizebox{0.85\textwidth}{!}{
\begin{tabular}{lcccccccc}
\toprule
Benchmark & \multicolumn{4}{c}{VQAv2} & \multicolumn{4}{c}{AdVQA} \\
MLLM & \multicolumn{2}{c}{InternVL} & \multicolumn{2}{c}{QwenVL} & \multicolumn{2}{c}{InternVL} & \multicolumn{2}{c}{QwenVL} \\
Metric & AUROC & AURAC & AUROC & AURAC & AUROC & AURAC & AUROC & AURAC \\
\midrule
Random Masking & 80.2 & 83.6 & 77.9 & 81.6 & 75.7 & 66.7 & 71.9 & 61.5 \\
Gaussian Noise & 80.3 & 84.0 & 77.8 & 80.7 & 75.5 & 66.5 & 73.4 & 63.0 \\
Reversed Masking & 77.6 & 81.5 & 76.1 & 79.2 & 74.8 & 65.0 & 71.9 & 61.3 \\
Ours (RoI Masking) & \textbf{81.3} & \textbf{85.2} & \textbf{78.5} & \textbf{81.9} & \textbf{76.8} & \textbf{66.9} & \textbf{74.5} & \textbf{64.8} \\
\bottomrule
\end{tabular}
}
\end{sc}
\end{small}
\vskip -0.1in
\end{center}
\end{table}

\subsection{Detector Failure Sensitivity}
\label{sec:detector_failure_sensitivity}

We next test how performance changes when RoI quality is deliberately degraded. Starting from GT boxes on GQA, we replace an increasing fraction of boxes with random crops. We evaluate and show the uncertainty detection performance in Table~\ref{tab:gqa_failure_sensitivity}.

\begin{table}[!h]
\begin{center}
\caption{\textbf{GQA detector-failure sensitivity.} A higher failure ratio means that more GT RoIs are replaced by random crops.}
\label{tab:gqa_failure_sensitivity}
\vskip 0.1in
\begin{small}
\begin{sc}
\resizebox{0.6\textwidth}{!}{
\begin{tabular}{lcc}
\toprule
Failure Ratio (\%) & InternVL AUROC & QwenVL AUROC \\
\midrule
0 & 72.1 & 66.3 \\
20 & 69.4 & 62.8 \\
40 & 67.1 & 61.4 \\
60 & 65.1 & 59.7 \\
80 & 63.3 & 58.7 \\
100 & 61.0 & 57.5 \\
\bottomrule
\end{tabular}
}
\end{sc}
\end{small}
\vskip -0.1in
\end{center}
\end{table}

The AUROC drops steadily as valid RoIs are replaced, confirming that CIM relies on preserving causal evidence rather than arbitrary perturbations such as random cropping. Meanwhile, In the actual pipeline, the detector fallback rate is low across benchmarks, as shown in Table~\ref{tab:grounding_failure_rate}.

\begin{table}[!h]
\begin{center}
\caption{\textbf{GroundingDINO fallback rate.} Fallback means no valid RoI is detected and the central crop is used.}
\label{tab:grounding_failure_rate}
\vskip 0.1in
\begin{small}
\begin{sc}
\resizebox{0.55\textwidth}{!}{
\begin{tabular}{lcccc}
\toprule
Benchmark & VQAv2 & OKVQA & AdVQA & MME \\
\midrule
Failure (\%) & 1.10 & 0.62 & 1.52 & 1.54 \\
\bottomrule
\end{tabular}
}
\end{sc}
\end{small}
\vskip -0.1in
\end{center}
\end{table}

\subsection{Relational and Global GQA Questions}
\label{sec:gqa_rel_global_analysis}

A concern with RoI cropping is that it may remove relationships or scene context. Our implementation uses the union box over detected query-relevant objects, so the spatial region between objects is preserved. We evaluate this directly on GQA by comparing common questions with relational/global questions.

\begin{table}[!h]
\begin{center}
\caption{\textbf{GQA relational/global analysis.} AUROC on common questions and complex questions involving relational or global reasoning.}
\label{tab:gqa_rel_global}
\vskip 0.1in
\begin{small}
\begin{sc}
\resizebox{0.65\textwidth}{!}{
\begin{tabular}{lcccc}
\toprule
Setting & \multicolumn{2}{c}{GQA-common} & \multicolumn{2}{c}{GQA-rel-global} \\
Metric & InternVL & QwenVL & InternVL & QwenVL \\
\midrule
SE & 78.4 & 69.2 & 79.2 & 69.4 \\
Ours ($K+K'$) & \textbf{79.9} & \textbf{70.2} & \textbf{80.2} & \textbf{71.5} \\
\bottomrule
\end{tabular}
}
\end{sc}
\end{small}
\vskip -0.1in
\end{center}
\end{table}

From the results in Table~\ref{tab:gqa_rel_global}, we can see the gain is still maintained when testing relational/global questions, suggesting that the union box in our implementation keeps enough contextual structure for CIM to remain effective.

\subsection{Alternative Grounding with Florence-2}
\label{sec:florence_detector}

CIM does not require GroundingDINO or a union box specifically. To test whether the framework is sensitive to the detector choice, we replace GroundingDINO with the query-conditioned Florence-2 detector~\cite{xiao2024florence}.

\begin{table}[!h]
\begin{center}
\caption{\textbf{Query-conditioned detector analysis.} Results are AUROC/AURAC. GD denotes GroundingDINO and FL denotes Florence-2.}
\label{tab:florence_detector_ablation}
\vskip 0.1in
\begin{small}
\begin{sc}
\resizebox{0.8\textwidth}{!}{
\begin{tabular}{lcccc}
\toprule
Benchmark & InternVL (GD) & InternVL (FL) & QwenVL (GD) & QwenVL (FL) \\
\midrule
VQAv2 & 81.3/85.2 & \textbf{82.2}/\textbf{85.9} & \textbf{78.4}/81.5 & 77.2/\textbf{82.6} \\
AdVQA & 76.8/66.9 & \textbf{77.7}/\textbf{68.3} & \textbf{74.5}/64.8 & 73.5/\textbf{65.0} \\
\bottomrule
\end{tabular}
}
\end{sc}
\end{small}
\vskip -0.1in
\end{center}
\end{table}

In the results shown in Table~\ref{tab:florence_detector_ablation}, the changes are small, indicating that CIM mainly depends on preserving query-relevant evidence and is not tied to a particular RoI detector.

\subsection{Empirical Validations for Energy Assumptions}
\label{sec:energy_shift_scale}

Practical causal energy could shift and spurious elements might remain. To further verify the influence of the decomposition assumption empirically, we evaluated logit causal energy by comparing our crops against vacant-background views whose area outside the crop is masked to black. We estimate causal and non-causal logit-energy shifts and compare their median magnitudes with the median top-1/top-2 logit gap. The results in Table~\ref{tab:causal_energy_shift} show that the shift is relatively small, validating the decomposition accuracy.

\begin{table}[!h]
\begin{center}
\caption{\textbf{Causal-energy shift scale.} Median $\|\Delta_c\|$ compared with median logit top-1/top-2 gap.}
\label{tab:causal_energy_shift}
\vskip 0.1in
\begin{small}
\begin{sc}
\resizebox{0.55\textwidth}{!}{
\begin{tabular}{lcccc}
\toprule
Benchmark & \multicolumn{2}{c}{InternVL} & \multicolumn{2}{c}{QwenVL} \\
 & $\|\Delta_c\|$ & Logit Gap & $\|\Delta_c\|$ & Logit Gap \\
\midrule
OKVQA & 0.75 & 3.62 & 0.38 & 4.21 \\
AdVQA & 0.50 & 4.57 & 0.50 & 4.81 \\
\bottomrule
\end{tabular}
}
\end{sc}
\end{small}
\vskip -0.1in
\end{center}
\end{table}

To validate the bounded non-causal energy in Theorem~\ref{theorem:jensen}, we empirically evaluated its magnitude, i.e., the logit difference between original and masked views, on hallucinated samples. The results are shown in Table~\ref{tab:noncausal_energy_shift}. Note that this value is accurate under our causal-invariance assumption. For scale anchoring, we compared $|\Delta_s|$ against the median top-1/top-2 logit gap. The median $|\Delta_s|$ (0.50-1.00) is consistently much smaller than the median decision margin (3.62-4.81). Based on these results, the theoretical bound $\epsilon$ can be estimated at $\epsilon \approx 1.0$. This value is sufficiently small relative to the logit scale, confirming our theoretical approximation holds in practice.

\begin{table}[!h]
\begin{center}
\caption{\textbf{Non-causal shift scale.} Median $\|\Delta_s\|$ on hallucinated samples compared with median logit top-1/top-2 gap.}
\label{tab:noncausal_energy_shift}
\vskip 0.1in
\begin{small}
\begin{sc}
\resizebox{0.55\textwidth}{!}{
\begin{tabular}{lcccc}
\toprule
Benchmark & \multicolumn{2}{c}{InternVL} & \multicolumn{2}{c}{QwenVL} \\
 & $\|\Delta_s\|$ & Logit Gap & $\|\Delta_s\|$ & Logit Gap \\
\midrule
OKVQA & 0.75 & 3.62 & 0.50 & 4.21 \\
AdVQA & 1.00 & 4.57 & 0.62 & 4.81 \\
\bottomrule
\end{tabular}
}
\end{sc}
\end{small}
\vskip -0.1in
\end{center}
\end{table}

Both $\|\Delta_c\|$ and $\|\Delta_s\|$ are substantially smaller than the corresponding decision margins. These estimates support the small-shift approximation in Theorem~\ref{theorem:jensen} and suggest that causal-energy leakage remains limited when query-relevant regions are preserved.

\subsection{POPE Entropy Distributions}
\label{sec:pope_entropy}

Table~\ref{tab:pope_entropy} reports the average token-level entropy during decoding on POPE. QwenVL has a much sharper output distribution than InternVL, which helps explain why discrete semantic clustering can be less stable on QwenVL and why EED is particularly effective for this model.

\begin{table}[!h]
\begin{center}
\caption{\textbf{Average token-level entropy on POPE.}}
\label{tab:pope_entropy}
\vskip 0.1in
\begin{small}
\begin{sc}
\resizebox{0.48\textwidth}{!}{
\begin{tabular}{lccc}
\toprule
MLLM & Random & Popular & Adversarial \\
\midrule
InternVL-8B & 0.51 & 0.51 & 0.54 \\
Qwen2.5-VL-7B & 0.06 & 0.07 & 0.08 \\
\bottomrule
\end{tabular}
}
\end{sc}
\end{small}
\vskip -0.1in
\end{center}
\end{table}

\subsection{NLI Model Ablation}
\label{sec:nli_ablation}

\begin{table}[!h]
\begin{center}
\caption{\textbf{NLI model ablation.} $\Delta(\text{SD Qwen}, \text{SD DeBERTa})$ represents the metric value difference between "SD w/ Qwen-7B" and "SD w/ DeBERTa-small".}
\label{tab:nli_ablation}
\vskip 0.1in
\begin{small}
\begin{sc}
\resizebox{0.6\textwidth}{!}{
\begin{tabular}{lcccc}
\toprule
Metric & \multicolumn{2}{c}{VQAv2} & \multicolumn{2}{c}{OKVQA} \\
 & AUROC & AURAC & AUROC & AURAC \\
\midrule
SE w/ DeBERTa-small & 74.5 & 78.0 & 77.5 & 83.5 \\
SD w/ DeBERTa-small & 76.2 & 80.4 & 78.6 & 84.4 \\
SE w/ Qwen-7B & 76.5 & 79.5 & 78.6 & 83.8 \\
SD w/ Qwen-7B & \textbf{78.4} & \textbf{81.5} & \textbf{79.6} & 84.2 \\
$\Delta$(SD Qwen, SD DeBERTa) & 2.2 & 1.1 & 1.0 & -0.2 \\
\bottomrule
\end{tabular}
}
\end{sc}
\end{small}
\vskip -0.1in
\end{center}
\end{table}

\begin{figure}[!h]
\vskip 0.2in
\begin{center}
\centerline{
\hfill
\includegraphics[width=0.3\textwidth]{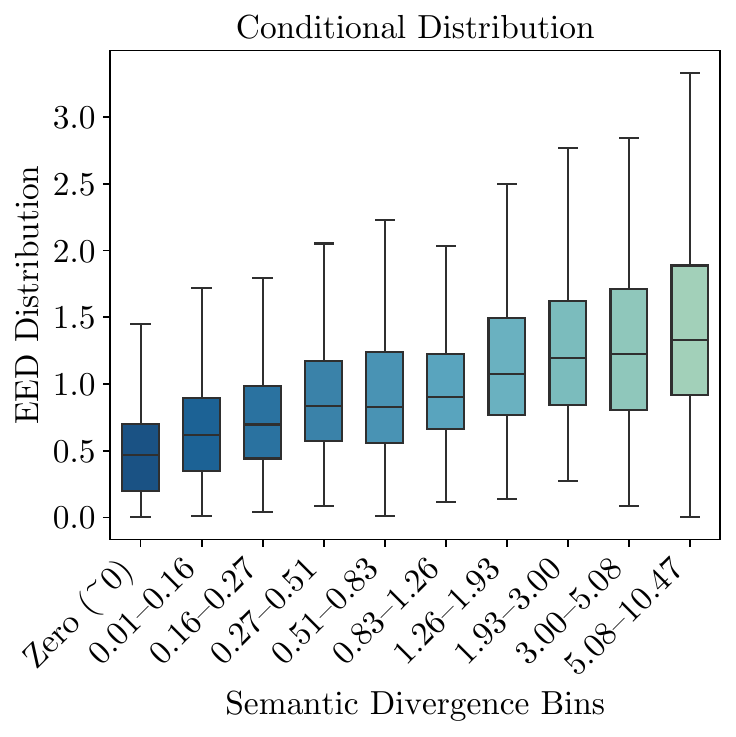}
\hfill
\includegraphics[width=0.31\textwidth]{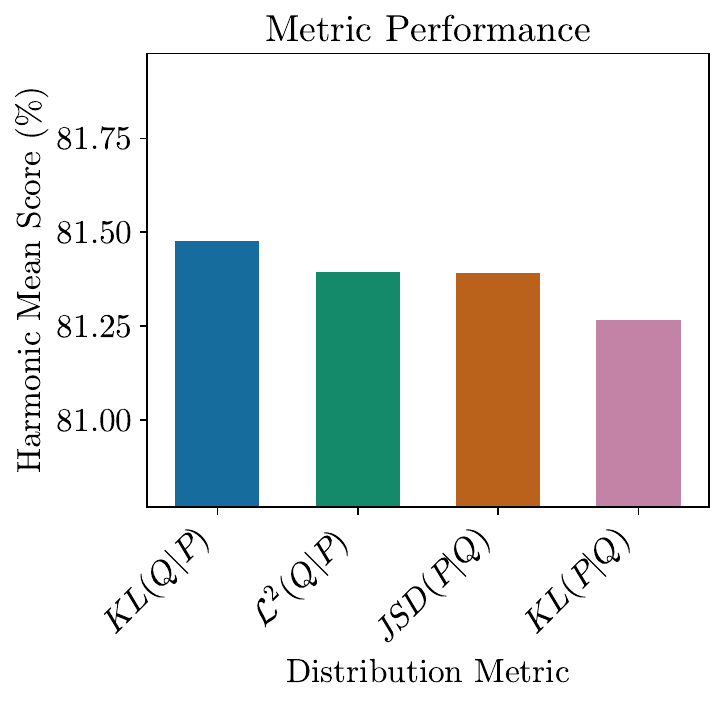}
\hfill
}
\caption{\textbf{Left:} Expected Embedding Drift (EED) as a proxy for Semantic Divergence. \textbf{Right:} Harmonic mean of AUROC and AURAC for different distribution distances, averaged across 2 MLLMs and 7 benchmarks.}
\label{fig:combined_eed_and_distances}
\end{center}
\vskip -0.2in
\end{figure}

Semantic clustering depends on an entailment model, so we compare the default Qwen-7B entailment model with a much smaller DeBERTa-small model~\cite{hedeberta}. Table~\ref{tab:nli_ablation} shows that the weaker entailment model slightly changes performance, but SD remains competitive and continues to outperform SE under the same NLI model.

\subsection{EED Proxy and Distance Metrics}
\label{sec:eed_distance_analysis}

Figure~\ref{fig:combined_eed_and_distances} left validates the connection between EED and Semantic Divergence by grouping samples according to $\mathcal U_{\text{sem}}$ and plotting the corresponding EED values. The mean EED increases with Semantic Divergence, supporting EED as an efficient proxy for semantic distribution shift.

Figure~\ref{fig:combined_eed_and_distances} right compares squared $L_2$ distance, Jensen-Shannon Divergence, reverse KL $D_{KL}(P\|Q)$, and forward KL $D_{KL}(Q\|P)$. The method is not highly sensitive to the exact distance, but forward KL performs best in our experiments, consistent with its role as an information-gain measure from the causally-filtered view to the original view.

\section{Licenses for Existing Assets}
\label{sec:asset_licenses}

We use existing public datasets, models, and detector assets for research evaluation only and do not redistribute their data or checkpoints. We cite the original creators in the main paper and follow the licenses or terms listed by their official project pages:
\begin{itemize}
    \item \textbf{VQAv2}~\cite{goyal2017making}: the VQA Consortium states that VQA annotations are licensed under Creative Commons Attribution 4.0 International, and its software is distributed under BSD-style redistribution terms; see \url{https://visualqa.org/terms.html}. The associated images are COCO images and remain subject to the original COCO/image-source terms.
    \item \textbf{OK-VQA}~\cite{marino2019ok}: the official download page provides OK-VQA annotations/questions and COCO images, but does not state a separate OK-VQA annotation license; see \url{https://okvqa.allenai.org/download.html}. We therefore use the benchmark under the official access terms and the underlying COCO image terms.
    \item \textbf{AdVQA}~\cite{li2021adversarial}: the official release page lists the AdVQA/AVQA data and image sources, but does not state a separate dataset license; see \url{https://adversarialvqa.org/download.html}. We use it under the official access terms and the source-image terms for COCO, Conceptual Captions, Fakeddit, and VCR as applicable.
    \item \textbf{GQA}~\cite{hudson2019gqa}: the official GQA website provides images that come from COCO and Flickr and scene graphs based on Visual Genome; see \url{https://cs.stanford.edu/people/dorarad/gqa/index.html} and \url{https://cs.stanford.edu/people/dorarad/gqa/download.html}. We could not find a separate license statement; the LMMs-Lab Hugging Face mirror lists MIT, but we treat the original COCO, Flickr, and Visual Genome terms as applicable to them.
    \item \textbf{POPE}~\cite{li2023evaluating}: the official RUCAIBox/POPE repository is released under the MIT license; see \url{https://github.com/RUCAIBox/POPE}. Its COCO-derived components also inherit the relevant COCO image/annotation terms.
    \item \textbf{MME}~\cite{fu2025mme}: the project page links to the Hugging Face dataset release, which lists Apache-2.0 as the dataset license; see \url{https://github.com/BradyFU/Awesome-Multimodal-Large-Language-Models/tree/Evaluation} and \url{https://huggingface.co/datasets/darkyarding/MME}.
    \item \textbf{InternVL2-8B}~\cite{chen2024far}: the official Hugging Face model card lists an MIT license and notes that its InternLM2.5-7B-Chat component is licensed under Apache-2.0; see \url{https://huggingface.co/OpenGVLab/InternVL2-8B}.
    \item \textbf{Qwen2.5-VL-7B-Instruct} and \textbf{Qwen2.5-7B-Instruct}~\cite{bai2025qwen25vltechnicalreport}: their official Hugging Face model cards list Apache-2.0; see \url{https://huggingface.co/Qwen/Qwen2.5-VL-7B-Instruct} and \url{https://huggingface.co/Qwen/Qwen2.5-7B-Instruct}.
    \item \textbf{GroundingDINO}~\cite{liu2024grounding}: the official repository license file is Apache-2.0; see \url{https://github.com/IDEA-Research/GroundingDINO/blob/main/LICENSE}.
    \item \textbf{Florence-2}~\cite{xiao2024florence}: the official Microsoft Hugging Face model card lists an MIT license; see \url{https://huggingface.co/microsoft/Florence-2-large}.
    \item \textbf{DeBERTa-small}~\cite{hedeberta}: the Microsoft DeBERTa repository and the \texttt{microsoft/deberta-v3-small} Hugging Face checkpoint list the MIT license; see \url{https://github.com/microsoft/DeBERTa/blob/master/LICENSE} and \url{https://huggingface.co/microsoft/deberta-v3-small}.
\end{itemize}

\end{document}